%% file: neurips_2024.tex
\documentclass{article}

\input{math_commands.tex}
\PassOptionsToPackage{numbers, sort, compress}{natbib}
\usepackage[final]{neurips_2024}

\usepackage{hyperref}

\usepackage[utf8]{inputenc} 
\usepackage[T1]{fontenc}    
\usepackage{hyperref}       
\usepackage{url}            
\usepackage{booktabs}       
\usepackage{amsfonts}       
\usepackage{nicefrac}       
\usepackage{microtype}      
\usepackage{xcolor}         

\usepackage{enumitem}
\usepackage{algorithm}
\usepackage[noend]{algorithmic}

\usepackage{amsmath}
\usepackage{amssymb}
\usepackage{mathtools}
\usepackage{amsthm}
\usepackage{bm}

\usepackage[capitalize,noabbrev]{cleveref}
\usepackage{multicol}

\theoremstyle{plain}
\newtheorem{theorem}{Theorem}[section]
\newtheorem{proposition}[theorem]{Proposition}
\newtheorem{lemma}[theorem]{Lemma}

\theoremstyle{definition}
\newtheorem{definition}[theorem]{Definition}
\newtheorem{assumption}[theorem]{Assumption}
\theoremstyle{remark}

\usepackage[textsize=tiny]{todonotes}
\usepackage{float}

\usepackage{xpatch}
\xpatchcmd\algorithmic{\leftmargin\labelwidth}{\leftmargin-\labelsep}{}{}
\newcommand{\paragraphbe}[1]{\noindent{\bf  #1 }}

\newcommand{\param}{\bm{\theta}}
\newcommand{\activeclients}{\mathcal{K}}

\newcommand{\updates}{\vd}
\newcommand{\noise}{\bm{\zeta}}

\title{Momentum Approximation in Asynchronous Private Federated Learning}

\author{%
  Tao Yu\thanks{Work was done while interning at Apple.}\\
  Amazon\\
  \And
  Congzheng Song \\
  Apple \\
  \And
  Jianyu Wang \\
  Apple \\
  \And
  Mona Chinits \\
  Apple \\
}

\begin{document}

\maketitle

\begin{abstract}
Asynchronous protocols have been shown to improve the scalability of federated learning (FL) with a massive number of clients. Meanwhile, momentum-based methods can achieve the best model quality in synchronous FL. However, naively applying momentum in asynchronous FL algorithms leads to slower convergence and degraded model performance. It is still unclear how to effective combinie these two techniques together to achieve a win-win. In this paper, we find that asynchrony introduces implicit bias to momentum updates. In order to address this problem, we propose momentum approximation that minimizes the bias by finding an optimal weighted average of all historical model updates.
Momentum approximation is compatible with secure aggregation as well as differential privacy, and can be easily integrated in production FL systems with a minor communication and storage cost. We empirically demonstrate that on benchmark FL datasets, momentum approximation can achieve $1.15 \textrm{--}4\times$ speed up in convergence compared to naively combining asynchronous FL with momentum.  
\end{abstract}

\input{sections/new_introduction}
\input{sections/background}

\input{sections/new_methodology}
\input{sections/experiments}
\input{sections/relatedwork}
\input{sections/conclusion}

\bibliography{reference}
\bibliographystyle{plain}


\appendix

\input{sections/new_appendix}


\newpage
\section*{NeurIPS Paper Checklist}

The checklist is designed to encourage best practices for responsible machine learning research, addressing issues of reproducibility, transparency, research ethics, and societal impact. Do not remove the checklist: {\bf The papers not including the checklist will be desk rejected.} The checklist should follow the references and follow the (optional) supplemental material.  The checklist does NOT count towards the page
limit. 

Please read the checklist guidelines carefully for information on how to answer these questions. For each question in the checklist:
\begin{itemize}
    \item You should answer \answerYes{}, \answerNo{}, or \answerNA{}.
    \item \answerNA{} means either that the question is Not Applicable for that particular paper or the relevant information is Not Available.
    \item Please provide a short (1–2 sentence) justification right after your answer (even for NA). 
\end{itemize}

{\bf The checklist answers are an integral part of your paper submission.} They are visible to the reviewers, area chairs, senior area chairs, and ethics reviewers. You will be asked to also include it (after eventual revisions) with the final version of your paper, and its final version will be published with the paper.

The reviewers of your paper will be asked to use the checklist as one of the factors in their evaluation. While "\answerYes{}" is generally preferable to "\answerNo{}", it is perfectly acceptable to answer "\answerNo{}" provided a proper justification is given (e.g., "error bars are not reported because it would be too computationally expensive" or "we were unable to find the license for the dataset we used"). In general, answering "\answerNo{}" or "\answerNA{}" is not grounds for rejection. While the questions are phrased in a binary way, we acknowledge that the true answer is often more nuanced, so please just use your best judgment and write a justification to elaborate. All supporting evidence can appear either in the main paper or the supplemental material, provided in appendix. If you answer \answerYes{} to a question, in the justification please point to the section(s) where related material for the question can be found.

IMPORTANT, please:
\begin{itemize}
    \item {\bf Delete this instruction block, but keep the section heading ``NeurIPS paper checklist"},
    \item  {\bf Keep the checklist subsection headings, questions/answers and guidelines below.}
    \item {\bf Do not modify the questions and only use the provided macros for your answers}.
\end{itemize}


\begin{enumerate}

\item {\bf Claims}
    \item[] Question: Do the main claims made in the abstract and introduction accurately reflect the paper's contributions and scope?
    \item[] Answer: \answerYes{} 
    \item[] Justification: The claims in the abstract and introduction are reflected in both theoretical and empirical results.

\item {\bf Limitations}
    \item[] Question: Does the paper discuss the limitations of the work performed by the authors?
    \item[] Answer: \answerYes{} 
    \item[] Justification: Limitations are discussed in Appendix~\ref{app:limitation}.

\item {\bf Theory Assumptions and Proofs}
    \item[] Question: For each theoretical result, does the paper provide the full set of assumptions and a complete (and correct) proof?
    \item[] Answer: \answerYes{} 
    \item[] Justification: Assumptions are provided in Section~\ref{section:ma} and proofs are provided in Appendix~\ref{app:proofs}.

    \item {\bf Experimental Result Reproducibility}
    \item[] Question: Does the paper fully disclose all the information needed to reproduce the main experimental results of the paper to the extent that it affects the main claims and/or conclusions of the paper (regardless of whether the code and data are provided or not)?
    \item[] Answer: \answerYes{} 
    \item[] Justification: Detailed description of hyperparameters and setup is provided in Appendix~\ref{sec:hyperparams}. We plan to open source the code and data for reproducing the results in the paper in the near future.

\item {\bf Open access to data and code}
    \item[] Question: Does the paper provide open access to the data and code, with sufficient instructions to faithfully reproduce the main experimental results, as described in supplemental material?
    \item[] Answer: \answerNo{} 
    \item[] Justification: We plan to open source the code and data for reproducing the results in the paper in the near future.

\item {\bf Experimental Setting/Details}
    \item[] Question: Does the paper specify all the training and test details (e.g., data splits, hyperparameters, how they were chosen, type of optimizer, etc.) necessary to understand the results?
    \item[] Answer: \answerYes{} 
    \item[] Justification: Detailed description of hyperparameters and setup is provided in Appendix~\ref{sec:hyperparams}. We also separately analyzed the impact of some critical hyperparameters in the algorithms.

\item {\bf Experiment Statistical Significance}
    \item[] Question: Does the paper report error bars suitably and correctly defined or other appropriate information about the statistical significance of the experiments?
    \item[] Answer: \answerNo{} 
    \item[] Justification: The statistical significance is not available momentarily as the experiments are expensive and time-consuming to run. We plan to include the error bars for all results in the next updated version.

\item {\bf Experiments Compute Resources}
    \item[] Question: For each experiment, does the paper provide sufficient information on the computer resources (type of compute workers, memory, time of execution) needed to reproduce the experiments?
    \item[] Answer: \answerYes{} 
    \item[] Justification: The resources are described in Appendix~\ref{sec:hyperparams}.
    
\item {\bf Code Of Ethics}
    \item[] Question: Does the research conducted in the paper conform, in every respect, with the NeurIPS Code of Ethics \url{https://neurips.cc/public/EthicsGuidelines}?
    \item[] Answer: \answerYes{} 
    \item[] Justification: The authors have reviewed and the research conform with the the NeurIPS Code of Ethics.

\item {\bf Broader Impacts}
    \item[] Question: Does the paper discuss both potential positive societal impacts and negative societal impacts of the work performed?
    \item[] Answer: \answerYes{} 
    \item[] Justification: The broader impacts of this paper is discussed in Appendix~\ref{app:impact}.
    
\item {\bf Safeguards}
    \item[] Question: Does the paper describe safeguards that have been put in place for responsible release of data or models that have a high risk for misuse (e.g., pretrained language models, image generators, or scraped datasets)?
    \item[] Answer: \answerNA{} 
    \item[] Justification: This question is not relevant for this paper.

\item {\bf Licenses for existing assets}
    \item[] Question: Are the creators or original owners of assets (e.g., code, data, models), used in the paper, properly credited and are the license and terms of use explicitly mentioned and properly respected?
    \item[] Answer: \answerYes{} 
    \item[] Justification: The authors have cited the original papers that produced the datasets.

\item {\bf New Assets}
    \item[] Question: Are new assets introduced in the paper well documented and is the documentation provided alongside the assets?
    \item[] Answer: \answerNA{} 
    \item[] Justification: This paper does not release new assets.

\item {\bf Crowdsourcing and Research with Human Subjects}
    \item[] Question: For crowdsourcing experiments and research with human subjects, does the paper include the full text of instructions given to participants and screenshots, if applicable, as well as details about compensation (if any)? 
    \item[] Answer: \answerNA{} 
    \item[] Justification: This paper does not involve crowdsourcing nor research with human subjects.

\item {\bf Institutional Review Board (IRB) Approvals or Equivalent for Research with Human Subjects}
    \item[] Question: Does the paper describe potential risks incurred by study participants, whether such risks were disclosed to the subjects, and whether Institutional Review Board (IRB) approvals (or an equivalent approval/review based on the requirements of your country or institution) were obtained?
    \item[] Answer: \answerNA{} 
    \item[] Justification: This paper does not involve crowdsourcing nor research with human subjects.

\end{enumerate}

\end{document}

%% file: math_commands.tex
\usepackage{amsmath,amsfonts,bm}

\def\eqref#1{equation~\ref{#1}}

\def\1{\bm{1}}

\def\vzero{{\bm{0}}}

\def\va{{\bm{a}}}

\def\vd{{\bm{d}}}
\def\ve{{\bm{e}}}

\def\vm{{\bm{m}}}

\def\vr{{\bm{r}}}

\def\vx{{\bm{x}}}
\def\vy{{\bm{y}}}

\def\mA{{\bm{A}}}

\def\mD{{\bm{D}}}
\def\mE{{\bm{E}}}

\def\mH{{\bm{H}}}
\def\mI{{\bm{I}}}

\def\mM{{\bm{M}}}

\def\mR{{\bm{R}}}

\def\mU{{\bm{U}}}
\def\mV{{\bm{V}}}
\def\mW{{\bm{W}}}
\def\mX{{\bm{X}}}

\def\mSigma{{\bm{\Sigma}}}

\DeclareMathAlphabet{\mathsfit}{\encodingdefault}{\sfdefault}{m}{sl}
\SetMathAlphabet{\mathsfit}{bold}{\encodingdefault}{\sfdefault}{bx}{n}

\def\gC{{\mathcal{C}}}

\def\gH{{\mathcal{H}}}

\def\gK{{\mathcal{K}}}

\def\gO{{\mathcal{O}}}

\newcommand{\E}{\mathbb{E}}

\newcommand{\R}{\mathbb{R}}

\newcommand{\normltwo}{L^2}

\DeclareMathAlphabet{\mathbbold}{U}{bbold}{m}{n}

%% file: sections/new_introduction.tex
\section{Introduction}
Practical deployment of synchronous federated learning (SyncFL)~\cite{mcmahan2017communication} encounters scalability issue due to the requirement on global synchronization of clients' model updates, wherein the central aggregation are contingent upon the completion of local training and communication across all participating clients. 
In order to address this issue, asynchronous FL (AsyncFL)~\cite{xie2019asynchronous,van2020asynchronous,park2021sageflow,chai2021fedat,nguyen2022federated,zhang2023no} was proposed, which allows concurrent model updates at both the server and the clients' side.
One concrete example is FedBuff~\cite{nguyen2022federated}, which is the state-of-the-art AsyncFL method and has been deployed in many production systems~\cite{huba2022papaya,wang2023flint}. 
In each FedBuff iteration, the server first broadcasts the global model and triggers local training on $K$ randomly sampled clients, then, receives clients' local model updates in a buffer. Once the buffer reaches a target cohort size $C \ll K$, the server will directly proceed to the next iteration without waiting for the all $K$ clients finish computation. 
As a result, the buffer gets filled up much quicker than SyncFL and the latency per iteration improves significantly~\citep{dutta2021slow}. 
However, since clients sampled in all previous iterations can contribute to the current global model update via stale gradients, AsyncFL methods typically have slower model convergence w.r.t iterations than SyncFL.

On the other hand, momentum-based optimizers such as momentum SGD and Adam have become dominant in the deep learning community due to their superior performance. Similar observations also appeared in SyncFL. For example, researchers found that applying momentum methods for the server model updates (e.g., FedAvgM~\cite{hsu2019measuring} and FedAdam~\cite{reddi2020adaptive}) can greatly improve the final model quality and convergence speed. 

Given the appealing benefits of asynchrony and momentum, it is desired to combine them to achieve a win-win in both efficiency and model quality. Unfortunately, the naive combination does not work. For instance, \cite{mitliagkas2016asynchrony,zhang2017yellowfin} showed that sophisticated tuning of the momentum parameter $\beta$ is very critical in asynchronous SGD (AsyncSGD). Rather than consistently using a large $\beta$ (e.g. 0.9) in the synchronous setting, a smaller or even negative $\beta$ is preferred and the best value may vary across datasets.
We observe the same phenomenon for AsyncFL. As shown in Figure~\ref{fig:implicit-momentum-results}, both FedAvgM and FedAdam with $\beta=0.9$ underperforms smaller $\beta$ in asynchronous setting while the pattern is the opposite if updates are synchronous.

\begin{figure}[t]
\centering
\includegraphics[height=0.265\linewidth]{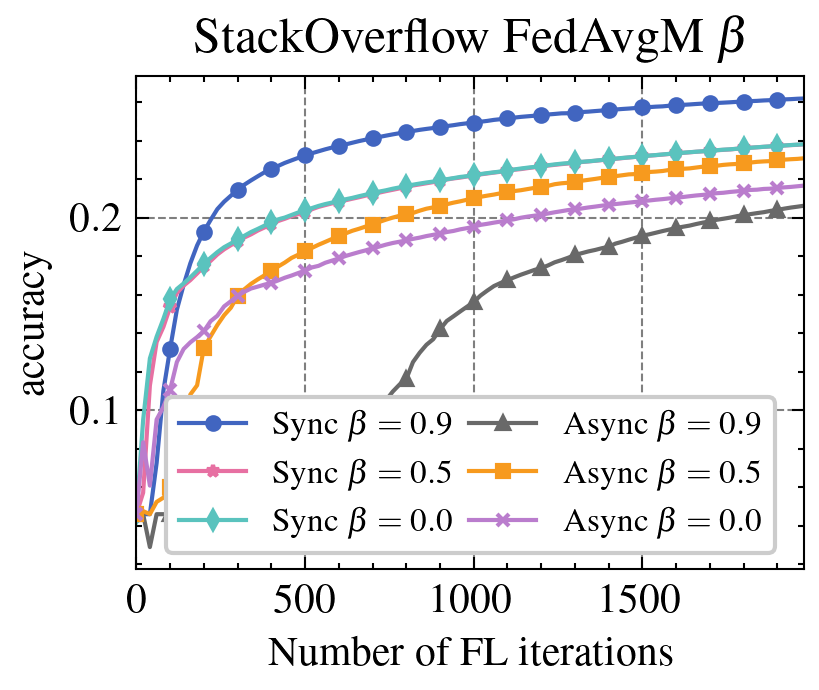}
\includegraphics[height=0.265\linewidth]{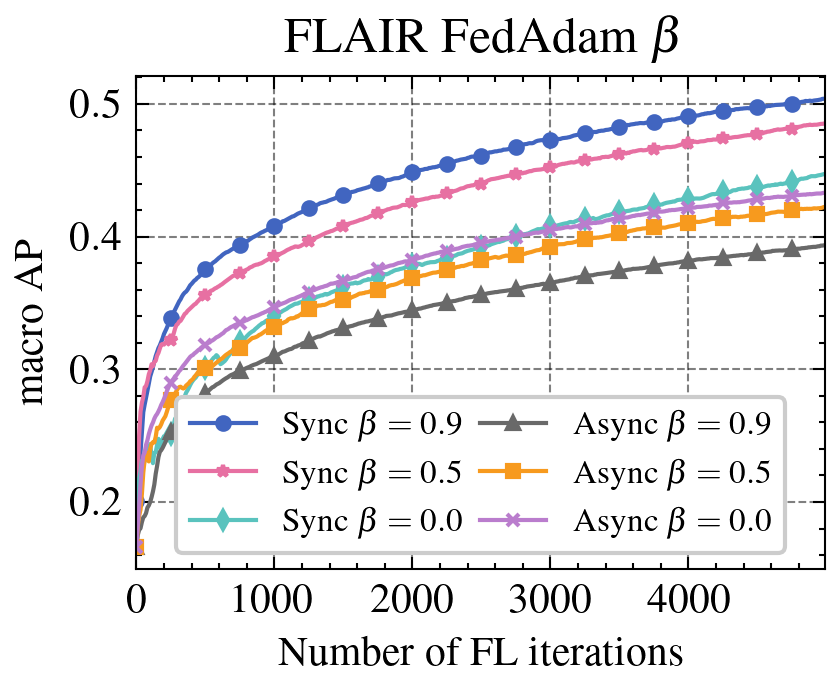}
\includegraphics[height=0.265\linewidth]{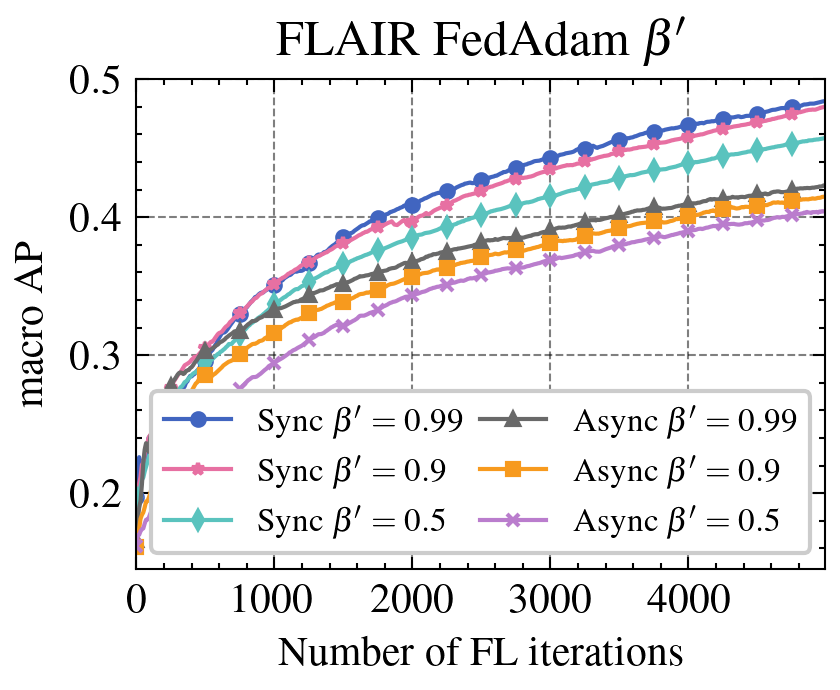}
\footnotesize
\caption{
(Left and Middle) In SyncFL, FedAvgM and FedAdam with momentum parameter $\beta=0.9$ converges fastest while it is not the case in AsyncFL: no momentum ($\beta=0$) or smaller $\beta=0.5$ is better.
(Right) The parameter $\beta^\prime$ for the second moments in FedAdam, on the other hand, has consistent impact on SyncFL and AsyncFL, i.e. larger $\beta^\prime=0.99$ is better.
}
\label{fig:implicit-momentum-results}
\end{figure}

Prior AsyncFL works proposed to down-scale the stale client updates before aggregation 
to control the impact of staleness~\cite{xie2019asynchronous,park2021sageflow}. 
However, this does not help in combining asynchrony and momentum. 
As shown in the experiments of FedBuff~\citep{nguyen2022federated}, even with lower weights for stale updates, $\beta$ still needs to be carefully tuned on different datasets and the best value can be $0$ (i.e. no momentum).
It remains an open question: \emph{is it possible to effectively integrate asynchrony and momentum in FL to simultaneously harness the advantages in scalability and better model quality?} 


\paragraphbe{Contributions.}  
In this paper, we provide an affirmative answer to the above question. Motivated by the fact that momentum method itself utilizes all past gradients by taking an exponential average (i.e., has unbounded staleness), we argue that the key issue in applying momentum to AsyncFL may not be the large staleness of model updates. Instead, the real problem is that naive asynchronous momentum method does not properly exploit the past information. We demonstrate that asynchrony introduces implicit weight bias to past gradients and hence, removes the momentum effect.

In order to address this problem, we propose a new algorithm named \emph{momentum approximation}, which solves a least square problem in each FL iteration $t$ to find the best coefficients to weight the historical model updates before $t$, such that the weighted historical updates are close to the momentum updates as in the synchronous setting and thus retain the acceleration from momentum. 
We highlight some key features of this algorithm: 
\begin{itemize}[leftmargin=3mm,topsep=0pt]
\item Momentum approximation is compatible with any momentum-based federated optimizer and its convergence pattern behaves similar to SyncFL. 
It resolves the need of extensively tuning $\beta$ from a wider range for different tasks in prior works. One can consistently set $\beta$ found in SyncFL to get the best model quality in the AsyncFL.
\item Momentum approximation can be easily integrated in production FL systems, and inherits all the benefits from FedBuff, such as its scalability, robustness and compatibility to privacy. 
It incurs only a minor communication of a iteration number in addition to model updates, and storage cost of historical updates on the server side. 
\item We empirically demonstrate that on two large-scale FL benchmarks, StackOverflow~\citep{stackoverflow} and FLAIR~\citep{song2022flair}, momentum approximation achieves $1.15 \textrm{--}4\times$ speed up in convergence and $3\textrm{--}20\%$ improvements in utility compared to vanilla FedBuff with momentum. 
\end{itemize}

%% file: sections/background.tex
\section{Background}
In federated learning (FL), we aim to train a model $\param\in\R^d$ with $m$ clients collaboratively.
In iteration $t$ of FL, a cohort of clients is sampled and the server broadcasts the current global model $\param_t$ to the sampled clients $\activeclients_t$. 
Each sampled client $k$ trains on their local dataset, and then submits the model updates $\Delta_k(\param_t)$ before and after the local training back to the server. 
In SyncFL, the server waits for the local model updates from all $K = |\activeclients_t|$ clients and uses the \emph{aggregated model updates}
$\updates_t=\frac{1}{K}\sum_{k \in \activeclients_t}\Delta_k(\param_t)$
to update the global model before proceeding to the next iteration. 
More formally, synchronous federated averaging (FedAvg)~\cite{mcmahan2017communication} algorithm updates the global model as $\param_{t+1} = \param_t - \eta \updates_t$
where $\eta$ denotes the server learning rate.


\paragraphbe{Momentum-based optimizers.}
In practice, momentum-based optimizers~\cite{wang2019slowmo,hsu2019measuring,reddi2020adaptive} on the server side are often more preferred than FedAvg as they can either greatly accelerate convergence or improve the final model quality given a fixed iteration budget. We denote these optimizers as \textsc{ServerOpt}, and the update rule of which can be formulated as:
\begin{align}
\vm_t = \beta\vm_{t-1} + (1-\beta)\updates_t,\,
\param_{t+1}=\param_t - \eta \mH_t^{-1}\vm_t \label{eq:momentum},
\end{align}
where $\beta\in [0, 1)$ is the momentum parameter and $\vm_t$ is the momentum buffer.
$\mH_t^{-1}$ is the preconditioner where $\mH_t=\mI_d$ in FedAvgM~\cite{hsu2019measuring}, and $\mH_t$ is the square root of accumulated or exponential moving average of $\updates_t$'s second moments in adaptive optimizers such as FedAdaGrad and FedAdam~\cite{duchi2011adaptive,kingma2014adam,reddi2020adaptive}. 

\paragraphbe{FL with differential privacy.}
Though the clients' raw data is never shared with the server in FL, the model updates $\Delta_k$ can still reveal sensitive information~\cite{melis2019exploiting,zhu2019deep, nasr2019comprehensive}. 
Differential privacy (DP) is a standard approach to prevent leakage from $\Delta_k$ and provide a meaningful privacy guarantee.
\label{sec:dp}
\begin{definition}[Differential Privacy \cite{dwork2006calibrating}]
A randomized algorithm $\mathcal{A}:\mathcal{D}\mapsto\mathcal{R}$ is ($\epsilon,\delta$)-differentially private, if for any pair of neighboring training populations $\mathcal{D}$ and $\mathcal{D^\prime}$ and for any subset of outputs $\mathcal{S}\subseteq\mathcal{R}$, it holds that
\begin{equation}
	\mathrm{Pr}[\mathcal{A}(\mathcal{D})\in\mathcal{S}] \leq e^\epsilon\cdot \mathrm{Pr}[\mathcal{A}(\mathcal{D^\prime})\in\mathcal{S}] + \delta.
\label{eq:dp}
\end{equation}
\end{definition}
We consider client-level DP where a training population $\mathcal{D^\prime}$ is the neighbor of $\mathcal{D}$ if $\mathcal{D^\prime}$ can be obtained by adding or removing one client from $\mathcal{D}$, and vice versa.
Gaussian Mechanism~\cite{dwork2006our} can be easily combined with FL~\cite{mcmahan2018learning} to enable DP, where two more additional steps are required in each iteration: (1) each client model update is clipped by $\mathrm{clip}(\Delta_k, S)=\Delta_k\cdot\min(1, S / \lVert\Delta_k\rVert_2)$ with $\normltwo$ sensitivity bound $S$, and (2) the aggregated clipped model updates are added with Gaussian noise as $\sum_k\mathrm{clip}(\Delta_k, S) + \mathcal{N}(0, \sigma^2 S^2\mI)$ where $\sigma$ is calibrated from a standard privacy accountant such as R\'enyi DP~\cite{mironov2017renyi}. 
We also assume a secure aggregation protocol is used so that the server learns only the sum $\sum_k \Delta_k$ but never the individual model updates $\Delta_k$~\cite{bonawitz2017practical,huba2022papaya,talwar2023samplable}.

%% file: sections/new_methodology.tex
\section{Applying Momentum to Asynchronous FL}
\label{section:ma}

In this section, we first demonstrate the problem in naively combining momentum and AsyncFL, and then introduce momentum approximation to address it. 
Unless otherwise stated, we focus on the FedBuff algorithm, which is a general and state-of-the-art  AsyncFL algorithm. 


\paragraphbe{Notation and assumptions.}
For a matrix $\mA$, we use $\mA_{[i, :]}$ to denote $i$-th row, $\mA_{[:, j]}$ the $j$-th column, and $\mA_{[i,j]}$ the $(i,j)$-th entry of $\mA$.
We denote the staleness as $\tau(k)$ for client $k$, i.e. $k$ is sampled at iteration $t-\tau(k)$ and their updates is received at $t$.
We let $\ve_i\in\{0, 1\}^T$  denote the one-hot encoding vector where all entries are 0 except for the $i$-th entry being 1, and $\bm{1} = [1,1,\dots,1]^\top$ denote a vector with all ones.
We denote $\gK_t$ as the set of $K$ clients sampled at iteration $t$, and $\gC_t$ as the set of $C$ clients whose updates received by server at iteration $t$. 

To gain insights on the impact of momentum in AsyncFL, we  define  $\vd^\star_t = \frac{1}{m}\sum_{k=1}^m \Delta_k(\param_t)$,
i.e. the average of local model updates over all $m$ clients starting from the same point $\param_t$. 
We make the following assumptions throughout.

\begin{assumption}[Bounded Population Client Update]
For each iteration $t$, $\lVert \updates_t^\star \rVert^2_2 \leq S^2$.
\label{assumption:norm}
\end{assumption}

\begin{assumption}[Bounded  Global Dissimilarity]
For all clients $k\in[m]$ and for each iteration $t$, $\E_{k\sim [m]}\lVert \Delta_k(\param_t)-\vd^\star_t \rVert^2_2 \leq G^2$.
\label{assumption:var}
\end{assumption}


\begin{assumption}(Bounded Client Subset Sampling Error)
\label{assumption:}
For the sampled clients set $\gK_t$ in each iteration $t$, $\E_{k\sim\gK_t}\lVert\Delta(\param_t)_k-\updates_t^\star \rVert_2^2 \leq \lVert\vd_t-\vd^\star_t\rVert_2^2 + \rho ^2$.
\label{assumption:bias}
\end{assumption}

\begin{assumption}(Random Arrival Order of Sampled Clients)
\label{assumption:random}
For each iteration $t$ and $s\leq t$, each client $k\in\gK_{t,s}$ is a random sample from the set $\gK_s$.
\end{assumption}

Assumptions~\ref{assumption:norm} and~\ref{assumption:var} are common in the FL literature~\cite{wang2019adaptive,li2020federated, yang2021achieving, nguyen2022federated}.
Assumptions~\ref{assumption:norm} also trivially holds with DP.
Assumption~\ref{assumption:bias} is a natural extension given Assumption~\ref{assumption:var} and $\rho^2\leq G^2$. 
Assumption~\ref{assumption:random} is based on the fact the timing of client participation tends to be random among sampled clients. We justify Assumption~\ref{assumption:random} in detail in Appendix~\ref{app:discuss}. 

\begin{figure*}[t]
\centering
\includegraphics[width=\linewidth]{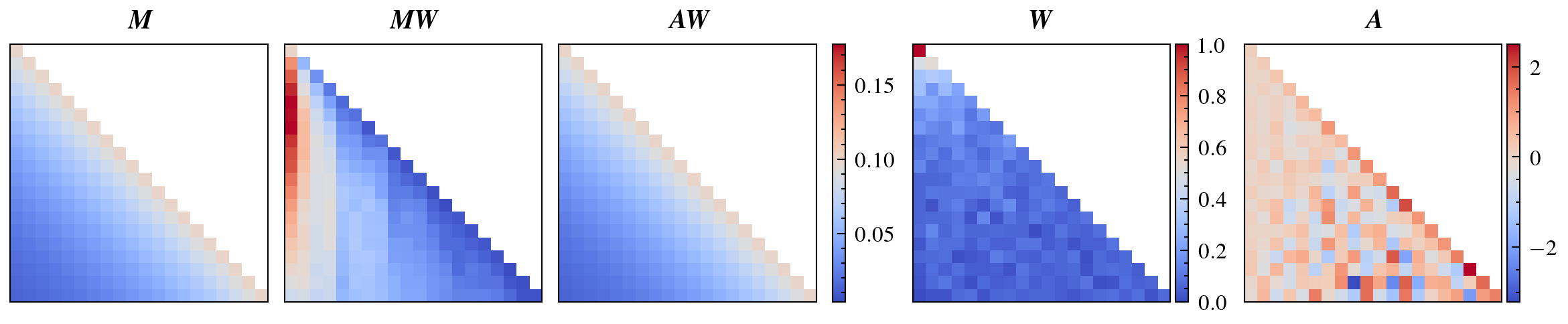}
\vspace{-1em}
\footnotesize
\caption{Visualization of the desired momentum matrix $\mM$ ($\beta=0.9$), the implicit momentum matrix $\mM\mW$, the approximated momentum matrix $\mA\mW$, the staleness coefficient matrix $\mW$, and the solved weighting matrix $\mA$ in momentum approximation.}
\label{fig:implicit-momentum}
\end{figure*}

\subsection{Implicit Momentum Bias}
In order to get a better understanding on the convergence issue of AsyncFL with momentum, we first present a general update rule for FL and then compare synchronous momentum methods and asynchronous ones as special cases.

Without loss of generality, we define $\bm{r}_t \in \mathbb{R}^d$ as the aggregated pseudo-gradient (or model updates) received by the server at iteration $t$, and denote $\bm{R} = [\bm{r}_1, \bm{r}_2, \ldots, \bm{r}_T] \in \mathbb{R}^{d \times T}$. Then, the following proposition holds~\cite{denisov2022improved}.
\begin{proposition}
    Suppose the server model $\param$ is updated using momentum method as follows:
    \begin{align*}
        \bm{m}_t &= \beta \bm{m}_{t-1} + (1-\beta) \bm{r}_t,\,
        \param_{t+1} = \param_t - \eta \bm{m}_t.
    \end{align*}
    This update rule is equivalent to $\param_{t+1} = \param_t - \eta (1-\beta) \sum_{s=1}^{t} \beta^{t-s} \bm{r}_s$. The final model after total $T$ iterations can be written as:
    \begin{align}
        \param_{T+1} = \param_1 - \eta \bm{R}\bm{M}^\top\bm{1}, \label{eqn:general_update}
    \end{align}
    where $\mM \in \mathbb{R}^{T \times T}$ is a lower-triangular matrix defined as:
        $
    \mM_{[t,s]} = \begin{cases}
    \beta^{t-s}(1-\beta) & \textnormal{if } t >= s\\
    0 & \textnormal{otherwise}
    \end{cases}.
    $
\end{proposition}
With the above general update rule, both SyncFL and AsyncFL can be treated as special cases. 
For SyncFL with all clients participating, the received (pseudo-) gradient on server is just the aggregated local model updates from all $m$ clients, that is, $\bm{r}_t = \bm{d}^\star_t$. 
For SyncFL with client sub-sampling, $\bm{r}_t = \bm{d}_t$.
Denote $\mD^\star = [\vd^\star_1, \vd^\star_2, \dots, \vd^\star_T]$ and $\mD=[\vd_1, \vd_2, \dots, \vd_T]$, then
\begin{align}
    \param_{T+1}^\star = \param_1 - \eta \mD^\star \mM^\top \bm{1}, \quad 
    \param_{T+1}^{\text{sync}} = \param_1 - \eta \mD \mM^\top \bm{1}. \label{eqn:sync_update}
\end{align}

For asynchronous setting, the concrete expression of $\bm{r}_t$ is more complicated. At each iteration of FedBuff, the server broadcasts the latest model $\param_t$ to $K$ random sampled clients to trigger local training and applies a global update after receiving $C$ local model updates from $\gC_t$. 
However, these received local model updates can be stale. 
Let $\gK_{t,s}$ be the set of clients sampled at iteration $s$ with updates received at iteration $t$ where $|\gK_{t,s}| = C_{t,s} \geq 0$ and $\sum_{s=1}^{t} C_{t,s}=C$. 
Then, the current (pseudo-) gradient on server can be written as a weighted average of all past updates:
\begin{align}
    \bm{r}_t  &= \frac{1}{C}\sum_{k\in\gC_t}(\tau(k) + 1)^{-p}\Delta_{k}(\param_{t-\tau(k)}) = \sum_{s=1}^{t}\frac{(\tau + 1)^{-p}}{C}\sum_{k\in\gK_{t,s}}\Delta_{k}(\param_{s}) \nonumber \\
    & = \sum_{s=1}^{t}(\tau + 1)^{-p}\frac{C_{t,s}}{C}(\vd^\star_s + \frac{1}{C_{t,s}}\sum_{k\in\gK_{t,s}}\Delta_{k}(\param_{s}) - \vd^\star_s) = \sum_{s=1}^t (\tau + 1)^{-p}  \frac{C_{t,s}}{C} (\bm{d}^\star_s + \noise_{t,s}), \label{eqn:async_r}
\end{align}
where $(\tau + 1)^{-p}$ is a down-scaling factor commonly used in AsyncFL to mitigate the impact of staleness $\tau=t-s$ on model updates~\cite{xie2019asynchronous,park2021sageflow,nguyen2022federated}, and $\bm{d}^\star_s$ has the same definition as of in the synchronous setting above.
Besides, since the server only receives a subset of $C_{t,s}$ individual local updates out of the sampled $K$ clients at iteration $s$, there is an extra sampling error denoted as $\noise_{t,s}$. We can define a weight matrix similar to $\mM$:
\[
\mW_{[t,s]}=\begin{cases}
(t-s + 1)^{-p} C_{t,s} / C  & \text{if } t >= s\\
0 & \text{otherwise}.
\end{cases}
\]

Then, one can easily derive that $\mR = \mD^\star\mW^\top + \mE$
where the $t$-th column of error matrix $\mE$ is defined as $\sum_{s=1}^t (\tau+1)^{-p}\frac{C_{t,s}}{C} \noise_{t,s}$. Substituting $\mR$ back into (\ref{eqn:general_update}), we get
\begin{align}
    \param_{T+1}^{\text{async}} 
    =& \param_1 - \eta \bm{D}^\star\mW^\top \bm{M}^\top\bm{1} -\eta \mE\mM^\top\bm{1} \nonumber \\
    =& \param_1 - \eta \mD^\star \left[\mM^\top + \underbrace{(\mM \mW- \mM)^\top}_{\text{implicit momentum bias}}\right]\bm{1} - \underbrace{\eta \mE \mM^\top \bm{1}}_{\text{async. sampling bias}} .\label{eqn:async_update}
\end{align}
Comparing the update rules \Cref{eqn:sync_update,eqn:async_update}, there are two additional terms in asynchronous setting. One is the implicit momentum bias: the algorithm implicitly assigns biased weights $\mM\mW$ (which is different from normal momentum weight $\mM$) to historical gradients, losing the benefits of momentum acceleration. We provide a visualization of $\mM$ and $\mM\mW$ in Figure~\ref{fig:implicit-momentum}. While the normal momentum assigns the largest weight to the most recent gradients, asynchronous momentum tends to weigh more towards stale gradients, as they arrive more frequently.
The second additional term in \Cref{eqn:async_update} is the asynchronous sampling bias: the server sampled $K$ clients from $s$-th iteration but can only received $C_{t,s}$ from the cohort at iteration $t$. 

Previous works~\cite{mitliagkas2016asynchrony} also observed that giving additional lower weights (e.g. set $p > 0$) to historical gradients can help convergence. This phenomenon can be intuitively explained by the definition of implicit momentum bias, which becomes smaller when $\mW$ approaches to the identity matrix.
However, this approach cannot entirely solve the problem. It is nearly impossible to set $\mW = \mI$ in realistic settings, as the current gradients may only arrive in future iterations.

\begin{theorem}
\label{theorem:async}
For SyncFL and AsyncFL with momentum, by choosing $\eta=\gO(\frac{1}{\sqrt{T}})$,
\begin{align}
\E\lVert\frac{1}{T}(\param_{T+1}^\star - \param_{T+1}^\textnormal{sync})\rVert^2_2 &\leq \frac{1}{2}\eta^2TG^2 = \gO(G^2), 
\\
\E\lVert\frac{1}{T}(\param_{T+1}^\star - \param_{T+1}^\textnormal{async})\rVert^2_2 &\leq \eta^2(2TS^2 + TG^2 + 2\frac{\rho^2}{C}) = \gO(S^2+G^2) \label{eq:async_bound}.
\end{align}
\end{theorem}
We defer the proof to Appendix~\ref{app:proofs}. 
Both SyncFL and AsyncFL have the same sampling bias in the order of $\gO(G^2)$ on average, but AsyncFL introduces an extra $\gO(S^2)$ term due to the implicit momentum bias.
Note that $\param_{T+1}^\star$ can be different for SyncFL and AsyncFL due to different parameter update trajectory, and we focus on comparing to the $\param_{T+1}^\star$ within each algorithm.

\subsection{Proposed Method: Momentum Approximation}
\label{sec:method}
From \Cref{eqn:async_r}, note that, in asynchronous setting, the received (pseudo-) gradient $\vr_t$ is already a weighted average of historical gradients. Therefore, instead of naively applying momentum on top of it, can we simply adjust the weights to imitate the momentum updates? Following this idea, we propose a new update rule for AsyncFL:
\begin{align}
    \param_{t+1} = \param_t - \eta \mR \va_t,
    \label{eq:ma-update}
\end{align}
where $\va_t\in\R^{T}$ is an arbitrary vector weighting the aggregated model updates. 
Accordingly, we have
\begin{align}
    \param_{T+1}^{\text{MA}}
    &= \param_1 - \eta \mR \mA^\top\bm{1}  = \param_1 - \eta \mD^\star \mW^\top \mA^\top \bm{1} - \eta \mE\mA^\top\bm{1} \nonumber \\
    &= \param_1 - \eta \mD^\star [\mM^\top + (\mA\mW - \mM)^\top]\bm{1} - \eta \mE \mA^\top\bm{1},
    \label{eq:ma-update-all}
\end{align}
where $\mA^\top=[\va_1, \va_2, \dots, \va_T]$.
One can choose a matrix $\mA$ such that $\mA\mW \approx \mM$. As a result, the implicit momentum bias is largely removed. 
The resulting algorithm approximates the synchronous momentum method without explicitly adjusting momentum. For this reason, we name the proposed method as \emph{momentum approximation (MA)}.

\paragraphbe{Implementation.}
The practical implementation of the proposed algorithm (outlined in Algorithm~\ref{alg:fedbuf}) is very straightforward. Thanks to the lower-triangular nature of both matrices $\mW$ and $\mM$, we can approximate the momentum matrix $\mM$ row-by-row, i.e., in an online fashion. At iteration $t$, the desired weights for past gradients are given as the $t$-th row of $\mM$ and known beforehand. We seek to optimize the following objective to find the best $\va_t = \va_{\text{opt}}$ to be used in \Cref{eq:ma-update}:
\begin{align}
    \min_{\va \in \R^{T}} \lVert \va^\top \mW - \mM_{[t,:]}\rVert_2^2 
    , \, \text{subject to } \va_{[s]} = 0, \forall s > t \label{eq:obj}.
\end{align}
In each vector $\va_t$, only the first $t$ elements are non-zero such that matrix $\mA$ is enforced to be an lower-triangular matrix. This is because the server cannot use gradients from future iterations. 
Solving \Cref{eq:obj} requires knowing $\mW_{[:t,:t]}$ (the first $t$ rows and columns of $\mW$) which can be obtained by having each received client $k$ to upload a one-hot encoding $\ve_s \in \{0,1\}^T$ of their model version $s$.\footnote{We need to send the one-hot encoding instead of the integer $t-\tau(k)$ to server as one-hot encoding is compatible with secure aggregation and DP to update $\mW$ and raw integer is not.}
More concretely, suppose at iteration $t$, the received updates at server are from a subset of $C$ clients and their model version are denoted as $\{t-\tau(k)\}_{k\in\gC_t}$. 
Then, the matrix $\mW$ is initialized with all 0 and updated online as $\mW_{[t,:]} = \frac{1}{C}\sum_{k\in\gC_t} \ve_{t-\tau(k)}^\top$.
Sending the extra $\ve_{t-\tau(k)}$ adds a negligible communication cost as $\ve_{t-\tau(k)}$ has a payload size of $T$ bits and $T\ll d$ for common FL tasks.

\begin{theorem}
\label{theorem:ma}
Under the condition that $\mW$ is full rank  and $\E\lVert\mA\rVert^2_F=\gO(CT^2)$, for AsyncFL with momentum approximation (MA), by choosing $\eta=\gO(\frac{1}{\sqrt{T}})$,
\begin{align}
\E\lVert\frac{1}{T}(\param_{T+1}^\star - \param_{T+1}^\textnormal{MA})\rVert^2_2 
\leq \eta^2(TG^2 + 2\frac{\rho^2}{TC}\E\lVert\mA\rVert^2_F) = \gO(G^2).
\label{eq:ma_bound_main}
\end{align}
\end{theorem}


We defer the proof and discuss the more general case when $\mW$ is not full rank to Appendix~\ref{app:proofs}.
Under the given condition, AsyncFL with momentum approximation achieves the same error as SyncFL and drops the implicit momentum bias term in~\Cref{eq:async_bound}.
We show that the condition of $\E\lVert\mA\rVert^2_F=\gO(CT^2)$ holds empirically in Appendix~\ref{app:discuss}.

\paragraphbe{Light-weight momentum approximation.}
The full approximation above requires a server-side storage cost of $\gO(Td)$ as all past received gradients $\mR$ needs to be saved to disk. 
This is usually not a concern as disk storage is cheap. In addition, $T$ in FL is typically in the order of thousands and the model size $d$ is small to meet on-device resource constraints~\cite{xu2023federated,xu2023training}.

In the case of both $T$ and $d$ are high and the disk storage cost becomes a concern, we propose a light-weight approximation which has no extra storage cost on the server. 
The light-weight update rule is the same as~\Cref{eq:ma-update} except that $\va_t$ is replaced by $\tilde{\va}_t$ defined recursively as below:
\begin{align}
\tilde{\va}_{t} &= u_t\ve_{t} + v_t \tilde{\va}_{t-1},
\label{eq:light-weight-vector}
\end{align}
where $u_t, v_t\in\R$ are to be optimized. 
With the recursive definition of $\tilde{\va}_t$, we rewrite~\Cref{eq:ma-update} as:
\begin{align}
\param_{t+1} &= \param_t - \eta \mR \tilde{\va}_t 
= \param_t - \eta (u_t\mR\ve_{t} + v_t \mR\tilde{\va}_{t-1}) = \param_t - \eta (u_t\vr_t + v_t \mR\tilde{\va}_{t-1}),
\label{eq:light-weight-rule}
\end{align}
which can be simplified to the following update rule similar to~\Cref{eq:momentum}:
\begin{align}
\tilde{\vm}_t = \mR\tilde{\va}_{t} = u_t\vr_t + v_t \tilde{\vm}_{t-1}, \,
\param_{t+1} &= \param_t - \eta\tilde{\vm}_t.
\label{eq:light-weight-rule-final}
\end{align}
The difference to~\Cref{eq:momentum} is that there are $T$ pairs of real numbers $(u_t, v_t)$ in light-weight MA instead of a single $\beta\in[0, 1]$. 
Since \Cref{eq:light-weight-rule-final} depends on $\tilde{\vm}$ and not on $\mR$, light-weight approximation saves the extra $\gO(Td)$ storage cost and has the same space complexity as the standard momentum updates by maintaining a single buffer $\tilde{\vm}_t$.

To find the best $(u_t, v_t) = (u_\text{opt}, v_\text{opt})$ in iteration $t$, we substitute~(\ref{eq:light-weight-vector}) back into~(\ref{eq:obj}):
\begin{align}
\min_{u, v\in\R} \lVert (u\ve_{t} + v\tilde{\va}_{t-1})^\top\mW - \mM_{[t,:]}\rVert_2^2. 
\label{eq:obj_light-weight}
\end{align}

\paragraphbe{Differentially private momentum approximation.}
Both the model updates $\Delta_k$ and the model version one-hot encoding $\ve_{t-\tau(k)}$ are sensitive information as they reveal the client's local data and their timing of participating FL.
We can use DP mechanisms to protect both information. 

Let $\gamma$ be a scaling factor and $\vy_k = \Delta_k \oplus \gamma\ve_{t-\tau(k)}\in\R^{d+T}$ be the payload that client $k$ intends to send to the server, where $\oplus$ denotes vector concatenation. 
We constrain the $\normltwo$ sensitivity of $\vy_k$ as $\bar{\vy}_k = \mathrm{clip}(\Delta_k, S_\Delta) \oplus \gamma\ve_{t-\tau(k)}$,
such that $\lVert\bar{\vy}_k\rVert_2 \leq \sqrt{S_\Delta^2 + \gamma^2} = S$.
Applying Gaussian Mechanism as $\sum_k\bar{\vy}_k + \mathcal{N}(0, \sigma^2 S^2\mI)$ satisfies $(\epsilon, \delta)$-DP as described in Section~\ref{sec:dp}.
By choosing  $\gamma=\frac{\sigma}{\sqrt{\xi^2 - \sigma^2}} S_\Delta$,
the Gaussian noise added to the un-scaled $\sum_k \ve_{t-\tau(k)}$ is $\mathcal{N}(0, (\frac{\sigma}{\gamma} S)^2\mI))$ with standard deviation:
\begin{align}
\frac{\sigma}{\gamma}S &= \frac{\sigma}{\gamma} \sqrt{S_\Delta^2+\gamma^2} = \sqrt{\frac{\sigma^2(\xi^2-\sigma^2)}{\sigma^2} + \sigma^2} = \xi.
\label{eq:xi}
\end{align}
In practice, we tune $\xi > \sigma$ to balance the utility on $\sum_k \Delta_k$ and $\sum_k \ve_{t-\tau(k)}$. 
As momentum approximation is a post-processing~\cite{dwork2014algorithmic} on the private aggregates:
\begin{align}
\vr_t = \frac{1}{C}[\sum_{k\in\gC_t}\mathrm{clip}(\Delta_k, S_\Delta) +\mathcal{N}(0, \sigma^2 S^2 \mI)],\, \mW_{[t,:]} = \frac{1}{C\gamma}[\sum_{k\in\gC_t}\gamma\ve^\top_{t-\tau(k)}+\mathcal{N}(0, \sigma^2S^2 \mI)],
\end{align} 
the MA update rules in~\Cref{eq:ma-update,eq:light-weight-rule-final} also satisfies the same $(\epsilon, \delta)$-DP guarantee.

\paragraphbe{Implicit momentum bias in the preconditioner.}
For adaptive optimizers such as Adam~\cite{kingma2014adam} and RMSProp, the preconditioner $\mH_t$ is the square root of exponentially decaying average of the gradients' second moments:
$
\mathrm{diag}(\mH_t)^2 \gets \beta^\prime\mathrm{diag}(\mH_t)^2 + (1-\beta^\prime) \vr_{t}^2.
$
The stale updates in $\vr_t$ bias the estimation of second moments similar to the implicit momentum bias term in the first moments. 
Nonetheless, preconditioner is known to be robust to delayed gradients~\cite{gupta2018shampoo,li2022differentially}. 
In Figure~\ref{fig:implicit-momentum-results}, we show that $\beta^\prime$ impacts the performance in AsyncFL similarly to in SyncFL.
We leave the study of implicit bias from staleness in the second moments to future work.

%% file: sections/experiments.tex
\begin{figure}[t]
\centering
\includegraphics[width=0.475\linewidth]{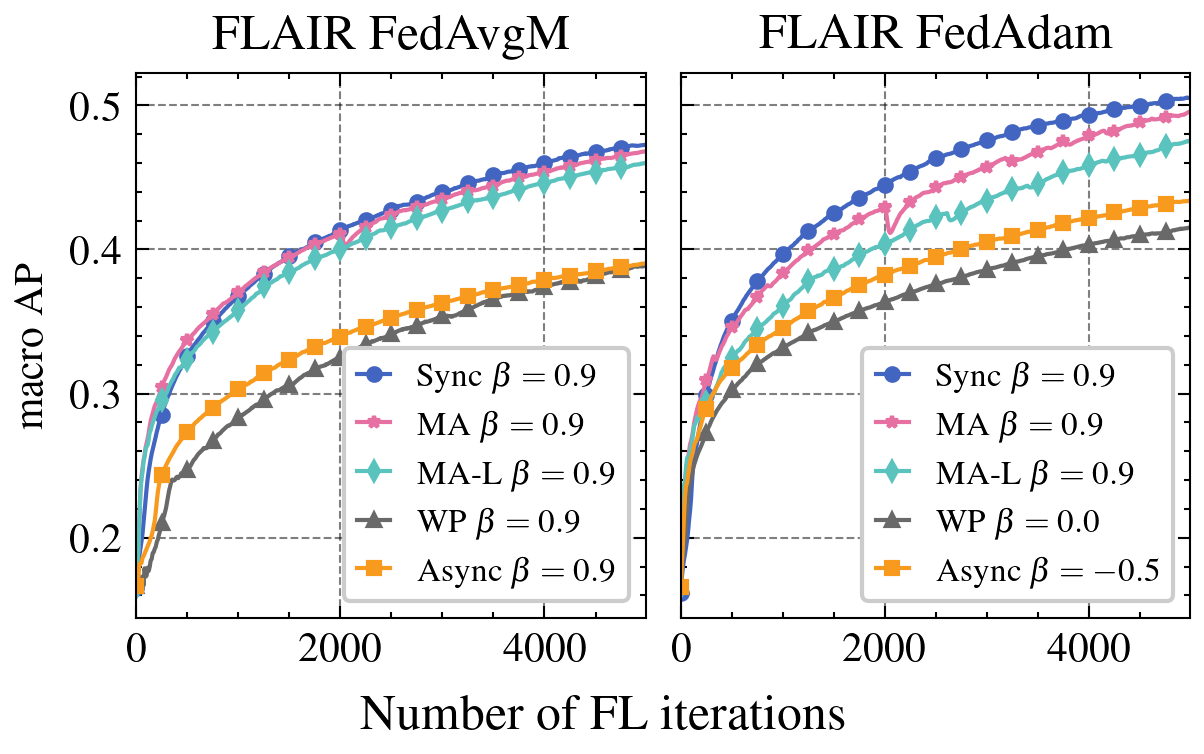}
\includegraphics[width=0.475\linewidth]{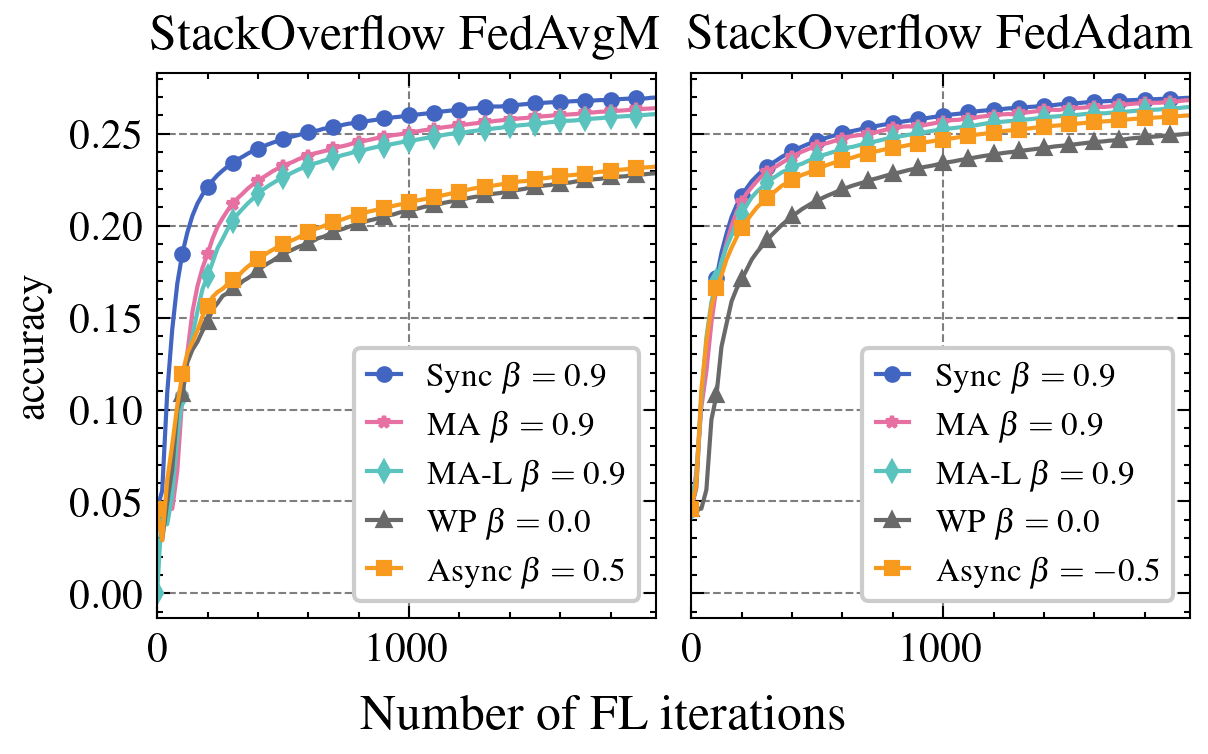}
\footnotesize
\caption{Comparison between MA, light-weight MA (MA-L) and baseline approaches.}
\label{fig:results}
\end{figure}

\begin{figure}[t]
\centering
\includegraphics[width=0.475\linewidth]{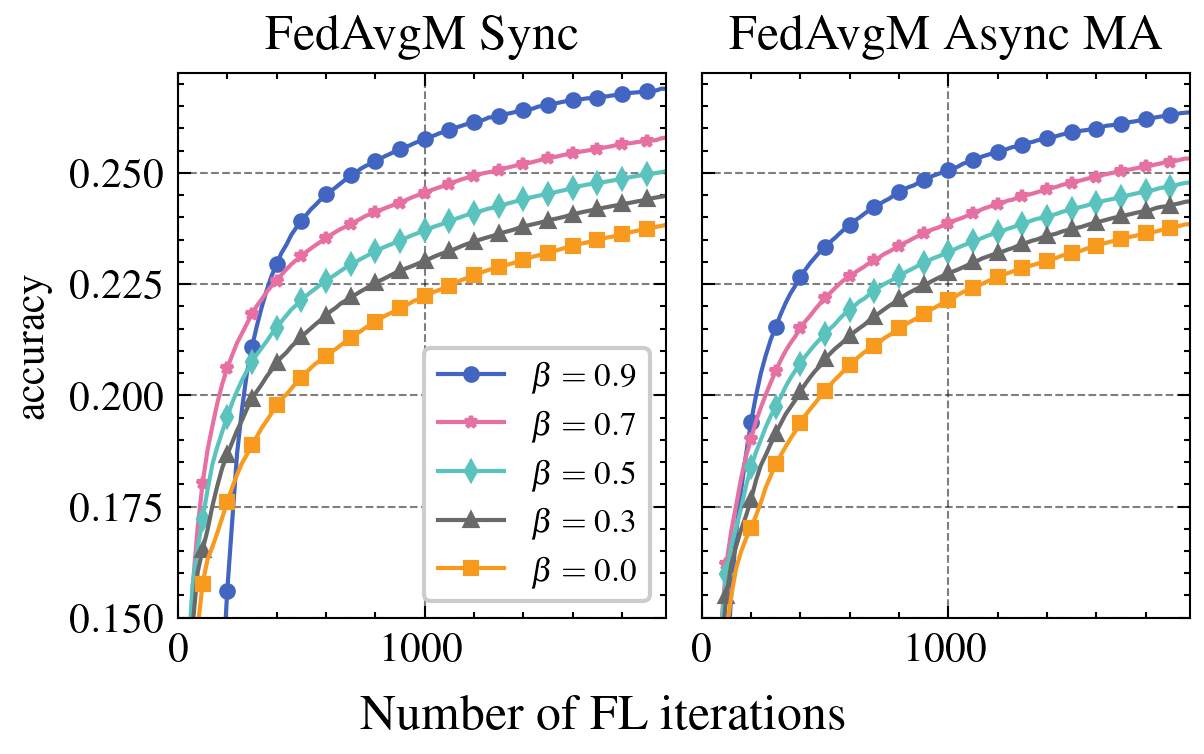}
\includegraphics[width=0.475\linewidth]{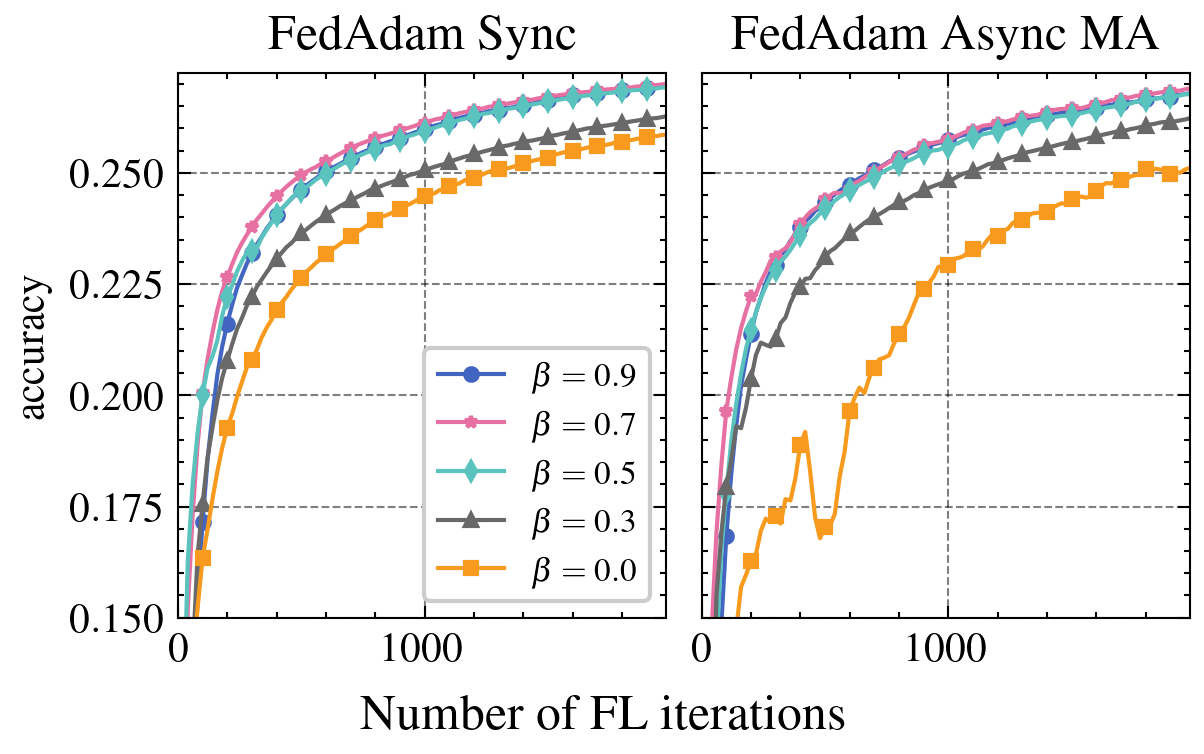}
\footnotesize
\caption{Impact of $\beta$ on SyncFL and AsyncFL with MA on the StackOverflow dataset.}
\label{fig:beta}
\end{figure}

\begin{figure}[t]
\centering
\includegraphics[width=0.475\linewidth]{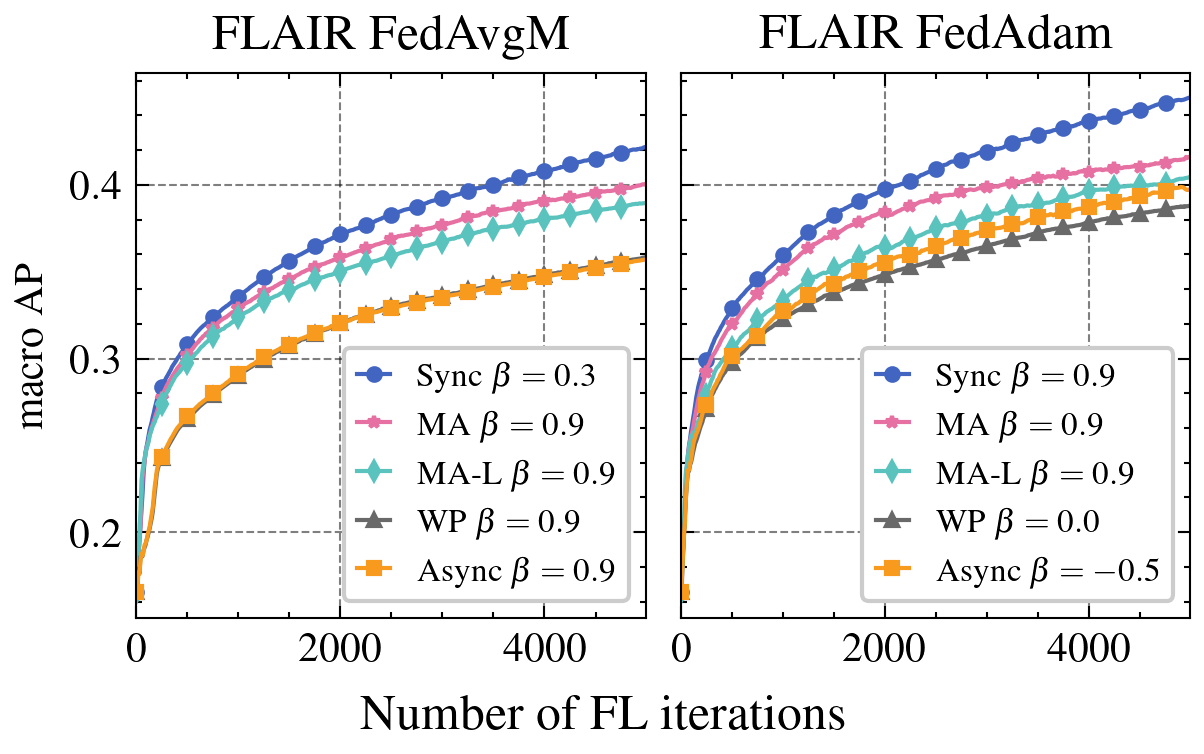}
\includegraphics[width=0.475\linewidth]{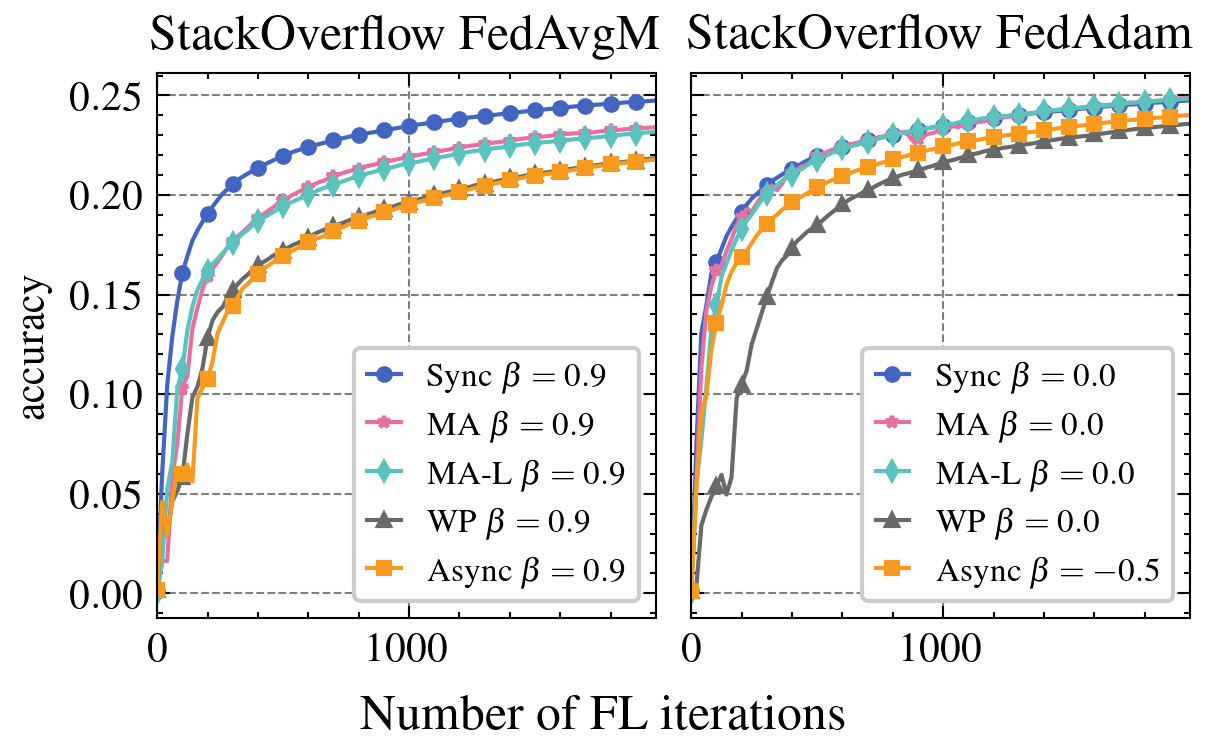}
\footnotesize
\caption{Comparison between MA, light-weight MA (MA-L) and baseline approaches with DP.}
\label{fig:dp-results}
\end{figure}

\section{Experiments}

In this section, we describe the empirical evaluation of momentum approximation (MA) with FedBuff.
We denote the light-weight MA in~\Cref{eq:obj_light-weight} as MA-light or MA-L.
We focus on two server-side momentum-based optimizers: FedAvgM~\cite{hsu2019measuring} and FedAdam~\cite{reddi2020adaptive}. 

\paragraphbe{Datasets and ML Tasks.} We conduct experiments on FLAIR~\cite{song2022flair}, a large-scale annotated image dataset for multi-label classification, and StackOverflow~\cite{stackoverflow}, a commonly used language modeling FL benchmark dataset. 
Both datasets have natural client partition which captures the non-IID characteristics in real world FL setting, and we believe they better represent the production FL datasets compared to other commonly used datasets (e.g. CIFAR10)
with artificially simulated client partition from a given distribution (e.g. Dirichlet~\cite{hsu2019measuring}).

For the FLAIR dataset, the task is to predict the set of coarse-grained labeled objects in a given image.
We use macro averaged precision (macro AP) as the evaluation metric.
For the StackOverflow dataset, the task is next word prediction and we use top prediction accuracy as the evaluation metric following prior work~\cite{reddi2020adaptive}. 
The details of hyerparameter choices are described in Appendix~\ref{sec:hyperparams}

\subsection{Baselines}

\paragraphbe{AsyncFL with tuned momentum parameter.}
We consider FedBuff as the baseline AsyncFL approach.
As suggested in~\cite{mitliagkas2016asynchrony,nguyen2022federated}, $\beta$ needs to be tuned carefully in AsyncFL and sometimes negative $\beta$ performs better. 
We tune $\beta$ from the range $(-1, 1)$.
 
\paragraphbe{Weight prediction (WP).}
WP is proposed to speed up AsyncSGD~\cite{kosson2021pipelined} and in particular, to address the implicit momentum issue~\cite{hakimi2019taming} in traditional distributed training setting. 
We modify WP to be compatible with AsyncFL as detailed in Appendix~\ref{sec:wp}.





\subsection{Results}
Figure~\ref{fig:results} summarizes the convergence comparison between our proposed  MA and the baseline approaches.
For both FedAvgM and FedAdam on both datasets, MA and MA-light significantly outperforms the baseline approaches of best tuned $\beta$.
We do not find WP worked well in the FL setting and acknowledge that more thoughtful integration is required to adopt techniques from AsyncSGD literature to AsyncFL, which is beyond the scope of this work.

\paragraphbe{Impact of $\beta$.}
Figure~\ref{fig:beta} illustrates how $\beta$ impacts SyncFL and AsyncFL with MA.
The correlation pattern between $\beta$ and the performance remains the same between SyncFL and AsyncFL with MA, i.e., larger $\beta$ leads to better performance in this task.
This demonstrates that the AsyncFL with MA can reuse the tuned $\beta$ in SyncFL experiments instead of searching in a wider range from scratch, which saves the costs from expensive hyperparameter tuning in FL.  

\paragraphbe{Impact of cohort size.} 
Smaller cohort can negatively impact the convergence of MA as the variance of $\vr_t$ in \Cref{eqn:async_r} increases.
We evaluate the impact of cohort size $C$ on the StackOverflow dataset by varying $C$ from 50 to 400. 
We compare the performance between SyncFL and FedBuff with MA on the same $C$.
Left of Table~\ref{tab:cohort_size} shows the impact of $C$ on MA. 
As $C$ increases from 50 to 400, the gap between AsyncFL with MA and SyncFL becomes smaller, which validates our hypothesis that smaller cohort size has more negative impacts on MA than SyncFL. 

\begin{table}[t]
\centering
\footnotesize
\caption{(Left) relative accuracy gap (\%) between SyncFL and AsyncFL with  MA for different cohort sizes $C$ on the StackOverflow dataset.
(Right) Relative speed up ($\times$) of MA compared to FedBuff baseline. FLR denotes FLAIR and SO denotes StackOverflow.}

\begin{tabular}{l|rr|rr}
\toprule
& \multicolumn{2}{c|}{FedAvgM} &  \multicolumn{2}{c}{FedAam} \\
$C$ & MA & MA-light & MA & MA-light \\
\midrule
50  & 5.17  &  7.68 & 7.26  &  11.94\\
100  & 3.17  &  4.25 & 2.47  &  4.51 \\
200  & 2.18  &  3.44 & 0.52  &  1.87 \\
400  & 1.39  &  2.8 & 0.46  &  0.4\\
\bottomrule
\end{tabular}
\quad
\begin{tabular}{l|cc|cc}
\toprule
& \multicolumn{2}{c|}{FedAvgM} &  \multicolumn{2}{c}{FedAam} \\
Setup & MA & MA-light & MA & MA-light \\
\midrule 
FLR & 3.56 & 3.01 & 2.20 & 1.66\\
FLR w. DP & 2.57 & 2.09 & 1.61 & 1.15 \\
SO & 3.96 & 3.30 & 1.62 & 1.30 \\
SO w. DP & 2.06 & 1.80 & 1.50 & 1.55 \\
\bottomrule
\end{tabular}
\label{tab:cohort_size}
\end{table}


\paragraphbe{DP results.}
Figure~\ref{fig:dp-results} illustrates the convergence results with DP. 
The pattern of the performance comparison is similar to that of the non-private case where both MA and MA-light outperform the AsyncFL baselines.
We notice that in FedAdam baseline, negative $\beta$ values are optimal on both datasets, indicating that the baseline requires more hyper-parameter tuning in a wider range of $\beta$. 


\paragraphbe{Speed up of MA.}
We finally evaluate the speed up of MA.
Following~\cite{nguyen2022federated}, we record the iterations that MA needed to achieve the best metric from FedBuff baseline, and have this number divided by the iterations the baseline took to get the relative speed up.
Right of Table~\ref{tab:cohort_size} summarizes the results. 
On both FLAIR and StackOverflow, MA speeds up FedAvgM more than FedAdam, and we suspect the reason is that the preconditioner in the FedAdam baseline is less affected, thereby mitigating the impact of staleness. Thus the improvement from MA is less significant. 
DP also impacts the speed up negatively as the estimation of $\mW$ becomes noisy and the noise scale on $\vr_t$ increases. 

%% file: sections/relatedwork.tex
\section{Related Work}
\paragraphbe{Asynchronous distributed SGD.}
The negative impact of gradient staleness has been studied in the traditional AsyncSGD setting.
A line of research focused on reducing the impact of staleness or the discrepancy between worker's model and the central model. 
\cite{zhang2016staleness} first proposed to down-scale the stale gradients based on their staleness $\tau$. 
\cite{barkai2019gap} argued that $\tau$ failed to accurately reflect the discrepancy and proposed to schedule down-scale based on the similarity between worker's model and the central model. 
\cite{zheng2017asynchronous} used a Taylor expansion and Hessian approximation to compensate for the staleness. 
\cite{hakimi2019taming,kosson2021pipelined} proposed parameter prediction of the future central model to reduce discrepancy, simply by following the optimization oracle for more iterations. 
The impact of staleness is exacerbated by momentum, as analyzed in~\cite{mitliagkas2016asynchrony}, where a small or even negative momentum parameter is preferred to adverse effects of asynchrony, which motivated an adaptive momentum training schedule~\cite{zhang2017yellowfin}.


\paragraphbe{Asynchronous federated learning.}
Existing AsyncFL works draw inspirations from AsyncSGD to handle stragglers and heterogeneous latency. 
\cite{xie2019asynchronous,park2021sageflow,nguyen2022federated} 
down-scaled the local model updates based on staleness $\tau$ before the central aggregation.
Nonetheless, as we showed, stale updates after down-scaling still affect training performance of momentum in practice.

On the other hand, FL differs from the traditional distributed SGD setting by having massive number of clients, higher communication overhead and more complicated system for secure aggregation and DP. 
Hence, many AsyncSGD algorithms do not directly fit in AsyncFL.
Gradient compensation~\cite{zheng2017asynchronous} applied complicated operations on individual client updates, which is less obvious how to implement with secure aggregation.
\cite{hakimi2019taming,chen2020asynchronous,kosson2021pipelined} required broadcasting model parameters and momentum, leading to doubled communication cost. 
Further research is needed to study how to integrate the other promising AsyncSGD algorithms efficiently in AsyncFL. 


Apart from aforementioned AsyncFL algorithms on weight aggregation, there are many orthogonal tactics on improving AsyncFL. 
For example,  
gradient compression techniques \cite{li2020efficient,lu2020privacy} improved communication efficiency; and model splitting \cite{chen2019communication,wang2021efficient,dun2023efficient} had each client responsible for training a certain part of the whole model.


\paragraphbe{Momentum-based federated optimizers.}
\cite{hsu2019measuring} first proposed to extend FedAvg with server-side momentum (FedAvgM) to accelerate convergence. 
\cite{reddi2020adaptive} improved the server-side optimization further by using adaptive optimizers with momentum.
\cite{khanduri2021stem,karimireddy2020mime,xu2021fedcm,ozfatura2021fedadc} proposed to perform momentum updates locally on the clients to alleviate local drift problem. 
\cite{sun2023role} introduced multistage FedGM which interpolates between FedAvg and FedAvgM with a hyperparameter scheduler, and provides a general momentum computation for FL to better control the momentum acceleration.
Our approach is compatible to any momentum-based optimizers and orthogonal to these works as they were not proposed to resolve the server-side implicit momentum bias in AsyncFL.

%% file: sections/conclusion.tex
\section{Conclusion}
We demonstrate how stale model updates incur an implicit bias in AsyncFL, which diminishes the acceleration from momentum-based optimizers.
To address this issue, we propose momentum approximation which optimizes a least square problem online to find the optimal weighted average of historical model updates that approximates the desired momentum updates.
Momentum approximation is easy to integrate in production FL systems with a minor storage and communication cost.
We empirically evaluate momentum approximation in both non-private and private settings on real-world benchmark FL datasets, and demonstrated that it outperforms the existing AsyncFL algorithms.

%% file: sections/new_appendix.tex
\newpage
\appendix
\onecolumn


\begin{algorithm}[t]
\caption{FedBuff with Momentum Approximation}
\label{alg:fedbuf}
\begin{algorithmic}
\STATE \textbf{Inputs:} client sampling rate $q$, cohort size $C$, server optimizer \textsc{ServerOpt}, server learning rate $\eta$, number of FL iterations $T$, client local learning rate $\eta_l$, number, number of client local SGD steps $Q$
\FOR{$t=1, \dots T$}
\STATE $\gK_t\gets$ sampled $K$ clients with sampling rate $q$
\STATE Run \textsc{Client}($\param_t, t$) for $k\in\gK_t$ asynchronously
\IF{receives $\Delta_k(\param_{t-\tau(k)})$ and $\ve_{t-\tau(k)}$ from  $k$}
\STATE $\vr_t \gets \vr_t + \frac{1}{C}\Delta_k(\param_{t-\tau(k)})$
\STATE $\mW_{[t,:]} \gets \mW_{[t,:]} + \frac{1}{C}\ve_{t-\tau(k)}^\top$
\ENDIF
\IF{server received $C$ local model updates}
\IF{light-weight momentum approximation}
\STATE $u_t, v_t\gets$ solve \Cref{eq:obj_light-weight}
\STATE $\tilde{\vm}_t\gets v_t\tilde{\vm}_{t-1} + u_t\vr_t$
\ELSE
\STATE $\mR_{[:,t]}\gets\vr_t$ update the pseudo-gradient history 
\STATE $\va_t\gets$ solve \Cref{eq:obj}
\STATE  $\tilde{\vm}_t\gets \mR\va_t$
\ENDIF
\STATE $\mH_{t}\gets$ update based on \textsc{ServerOpt}
\STATE $\param_{t+1}\gets\param_t - \eta \mH_t^{-1}\tilde{\vm}_t$
\STATE $\vr_{t+1}, \mW_{[t+1,:]} \gets \vzero$
\ENDIF
\ENDFOR
\STATE
\STATE \textbf{function} \textsc{Client}($\param, t$)
\begin{ALC@g}
\STATE $\param^\prime\gets$ run $Q$ SGD steps with $\eta_l$ on local data
\STATE $\ve_t\gets$ one-hot encoding of $t$
\STATE Upload $\Delta=\param - \param^\prime$ and $\ve_t$ to server
\end{ALC@g}
\end{algorithmic}
\end{algorithm}

\section{Additional Experiments Details}
\subsection{Experimental Setup}
\label{sec:hyperparams}

All experiments were run on a machine with 4 Nvidia A100 GPUs with 40GB VRAM. 
Each FLAIR experiment took 12 hours to finish on average and each StackOverflow experiment took 6 hours.

\paragraph{Client delay distribution.}
Following \cite{nguyen2022federated}, we adopt half-Normal distribution to model the client delay distribution. 
We demonstrate the impact of different distributions on MA in Appendix~\ref{app:additional}.

\paragraph{Staleness scaling and bounding.}
We tune the power in the down-scaling factor $p$ between 0.5 to 2.0 to control strength of the scaling. 
We further set a maximum staleness bound $\tau_{\max}$ (default to 20) and drop $\Delta_k$ if $\tau(k) > \tau_{\max}$. 

\paragraph{Hyperparameters.}
For the FLAIR dataset, we use a ResNet-18 model~\cite{he2016deep} following the setup in~\cite{song2022flair}. 
We train the model for 5,000 iterations with local learning rate set to 0.1, local epochs set to 2, and local batch size set to 16. For the StackOverflow dataset, we use a 3-layer Transformer model~\cite{vaswani2017attention} following the setup in~\cite{wang2021field}.
We train the model for 2,000 iterations with local learning rate set to 0.3 and local epochs set to 1, and local batch size set to 16. 
The cohort size $C$ is set to 200 for both dataset.
For FedAvgM, we search the learning rate $\eta$ between (0.1, 1.0).
For FedAdam, we set the $\beta^\prime$ for the second moment to 0.99 and the adaptivity parameter to 0.01, and search the server learning rate $\eta$ between (0.01, 0.1).

For DP experiments, we set the $(\epsilon,\delta)$ privacy budget to $(2.0, 10^{-7})$-DP with a simulated population size of $10^7$ and cohort size of 5,000 following prior work~\cite{mcmahan2018learning}. 
We set the $\normltwo$ clipping bound $S_\Delta$ to 0.1 for FLAIR and 0.2 for StackOverflow.
We use amplification by subsampling with R\'enyi DP to calibrate the Gaussian noise scale $\sigma$~\cite{mironov2017renyi,mironov2019r}.
Though we focus on independent Gaussian mechanism in each iteration, our approach is also compatible with DP-FTRL mechanisms with correlated noise between iterations~\cite{kairouz2021practical,choquette2023amplified}.
For momentum approximation where $\mW$ needs to be estimate privately, we set $\xi$ in \Cref{eq:xi} such that $S = 1.1 S_\Delta$, i.e. we pay 10\% extra noise on $\Delta$ to learn $\mW$ privately with the same budget.

For all experiments, we apply exponential moving average (EMA) on central model parameters $\theta$ with decay rate of 0.99~\cite{de2022unlocking}, and report the metrics evaluated on the EMA model parameters.

\subsection{Weight Prediction Baseline}
\begin{algorithm}[t]
\caption{FedBuff with Weight Prediction}
\label{alg:wp}
\begin{algorithmic}
\STATE \textbf{Inputs:} client Poisson sampling rate $q$, cohort size $C$, server optimizer \textsc{ServerOpt}, server learning rate $\eta$, number of FL iterations $T$, client local learning rate $\eta_l$, number, number of client local SGD steps $Q$, $\alpha$ EMA decay parameter for historical model updates
\WHILE{$t < T$}
\STATE $\gC_t\gets$ sampled clients with Poisson sampling rate $q$
\STATE Run \textsc{Client}($\theta_t, t, \eta\mH_t^{-1}\vx_h$) for $k\in\gC_t$ asynchronously
\IF{receives $\Delta_k(\theta_{t-\tau(k)})$ from  $k$}
\STATE $\vx_t \gets \vx_t + \frac{1}{C}\Delta_k(\theta_{t-\tau(k)})$
\ENDIF
\IF{received $C$ results in the buffer}
\STATE $\vm_t, \mH_{t}\gets$ update based on \textsc{ServerOpt}
\STATE $\theta_{t+1}\gets\theta_t - \eta \mH_t^{-1}\vm_t$
\STATE $\vx_h \gets \alpha\vx_h + (1-\alpha)\vx_t$
\STATE $\vx_{t+1} \gets \vzero$
\STATE $t\gets t + 1$
\ENDIF
\ENDWHILE
\STATE
\STATE \textbf{function} \textsc{Client}($\theta, t, \eta\mH_t^{-1}\vx_h$)
\begin{ALC@g}
\STATE $\tau\gets t^\prime - t$ gets the current staleness 
\STATE $\hat{\theta}_{t + \tau}\gets \theta - \tau \eta\mH_t^{-1}\vx_h$
\STATE $\theta^\prime\gets$ run $Q$ SGD steps with $\eta_l$ on $\hat{\theta}_{t + \tau}$
\STATE Upload $\Delta=\hat{\theta}_{t + \tau} - \theta^\prime$ to server
\end{ALC@g}
\end{algorithmic}

\end{algorithm}

\label{sec:wp}
WP is proposed to speed up asynchronous SGD~\cite{kosson2021pipelined} and in particular, to address the implicit momentum issue~\cite{hakimi2019taming}. 
In Algorithm~\ref{alg:wp}, we modify WP to be compatible with adaptive optimizer in FedBuff as another baseline for evaluating momentum approximation. 
To predict the future model, the server sends both $\theta_t$ and the historical model updates $\vx_h$ to devices. 
For a sampled client with staleness $\tau$, the client tries to first predict the future model $\hat{\theta}_{t+\tau}\approx {\theta}_{t+\tau}$, by running $\tau$ steps of $\textsc{ServerOpt}(\theta, \vx_h, \eta)$. 
We consider $\vx_h$ to be the exponential decay averaging of $\mX_{:t}$ for variance reduction.
For adaptive $\textsc{ServerOpt}$ such as Adam, we send $\mH^{-1}_t\vx_h$ to devices for WP.
Client then runs the local SGD steps on $\hat{\theta}_{t+\tau}$ and returns the model update to the server. 
Note that this method will also double the communication as the server needs to send extra historical model updates for WP.

\subsection{Additional Results}
\label{app:additional}
\begin{table}[t]
\centering
\caption{(Left) Accuracy (\%) with momentum approximation (MA) for different staleness bound $\tau_{\max}$ on the StackOverflow dataset.
(Right) Relative least square error in \Cref{eq:obj} for different client delay distribution.
}
\begin{tabular}{l|rr|rr}
\toprule
& \multicolumn{2}{c|}{FedAvgM} &  \multicolumn{2}{c}{FedAam} \\
$\tau_{\max}$ & MA & MA-light & MA & MA-light \\
\midrule
20 & 26.36  &  26.11 & 26.79  &  26.45 \\
30 & 26.26  &  26.05 & 26.54  &  26.24 \\
40 & 26.15  &  25.95 & 26.25  &  25.87 \\
50 & 25.97  &  25.90 & 26.00  &  25.46 \\
\bottomrule
\end{tabular}
\quad
\begin{tabular}{l|rr}
\toprule
Client Delay  &   &  \\
Distribution &  MA & MA-light \\
\midrule
Half-Normal & 2.58\% & 33.07\% \\
Uniform  & 8.35\% & 36.78\% \\
Exponential  & 2.41\% & 33.89\% \\
\bottomrule
\end{tabular}
\label{tab:delay}
\end{table}

\paragraph{Impact of staleness bound $\tau_{\max}$.} 
We empirically study the impact of $\tau_{\max}$ on momentum approximation by varying it from 20 to 50.
Left of Table~\ref{tab:delay} summarizes the results on the StackOverflow dataset, where larger $\tau_{\max}$ leads to lower accuracy. 
Another observation is that the drop in performance of FedAdam is greater than that of FedAvgM which could be from the impact of staleness on the estimation of preconditioner in FedAdam. 

\paragraph{Impact of client delay distribution.}
We evaluate the impact of different client delay distribution on momentum approximation objective in~\Cref{eq:obj}. 
We choose Half-Normal, Uniform and Exponential distribution following~\cite{nguyen2022federated}.
We measure the relative least square error as $\lVert \mA\mW - \mM \rVert_F^2 / \lVert\mM \rVert_F^2$ using the setup in Appendix~\ref{sec:hyperparams} and report the results in the right of Table~\ref{tab:delay}.
The relative error for light-weight approximation is much higher as expected and is more than 30\% for all three distributions.
The approximation error is worst for Uniform distribution, while it is in the similar range for Half-Normal and Exponential distribution.  
Uniform client delay distribution is unrealistic in production FL system and thus our method is robust to different sensible client delay distributions.

\section{Additional Details of Momentum Approximation}
\paragraph{Notation.}
Let $\gH_{i}$ be the history of sampled clients $\gK_1, \gK_2, \dots, \gK_i$ up to iteration $i$. 
Let $\noise_{s} = \updates_s - \updates_s^\star$ be the sampling error for the subset $\gK_s$ sampled at iteration $s$.  
Let $\updates_{t,s} = \frac{1}{C_{t,s}}\sum_{k\in\gK_{t,s}}\Delta_k(\param_s)$ be the averaged update of the set of clients $\gK_{t,s}\subseteq\gK_s$ whose updates arrived at iteration $t$ and denote $C_{t,s} = |\gK_{t,s}|$.
Let $\noise_{t,s} = \updates_{t,s} - \updates_s^\star$ be the sampling error for the subset $\gK_{t,s}$. 
Let $\gK_{:i,:i}$ be the subsets $\gK_{t,s}$ for all $1\leq s \leq t \leq i$ and $C_{:i, :i}$ be their cardinalities.

We first note an upper bound on the sampling error $\noise_{s}$. 
From Assumption~\ref{assumption:var}, 
\begin{align}
\E_{\gK_s}[\lVert\noise_{s}\rVert_2^2] &= \E_{\gK_s}[\lVert\frac{1}{K}\sum_{k\in\gK_{s}}\Delta_k(\param_s) - \updates_s^\star\rVert_2^2] \nonumber \\
& \leq \E_{\gK_s}[\frac{1}{K^2}K\sum_{k\in\gK_{s}}\lVert\Delta_k(\param_s) - \updates_s^\star\rVert_2^2] \nonumber \\ 
& = \frac{1}{K}\sum_{j=1}^K\E_{k\sim [m]}[\lVert\Delta_k(\param_s) - \updates_s^\star\rVert_2^2]\leq G^2
\end{align}

\subsection{Proof of Results}
\label{app:proofs}

We state two useful lemmas for proving Theorem~\ref{theorem:async} and~\ref{theorem:ma}. 

\begin{lemma}
\label{lemma:1}
For any $1 \leq s \leq t \leq i$ and $1 \leq s^\prime \leq t^\prime\leq i$,
\begin{align}
\E_{\gK_{t,s}, \gK_{t^\prime,s^\prime}|C_{t,s}, C_{t^\prime,s^\prime},\gH_i}[\noise_{t,s}^\top\noise_{t^\prime,s^\prime}] \leq \noise_{s}^\top\noise_{s^\prime} + \frac{\rho^2}{C_{t,s}}\mathbbold{1}[(t,s)=(t^\prime,s^\prime)].
\end{align}
\end{lemma}
\begin{proof} 
Let $\noise_{t,s,k} = \Delta_k(\param_s) - \updates^\star_s$ for $k\in\gK_{t,s}$, and given Assumption~\ref{assumption:random}, we have
\begin{align}
\E_{k\sim\gK_{t,s}|\gH_i}[\noise_{t,s,k}] = \E_{k\sim\gK_{s}|\gH_i}[\Delta_k(\param_s)] - \updates^\star_s = \updates_s - \updates^\star_s = \noise_s.
\end{align}
Then for the case when $(t,s) \neq (t^\prime, s^\prime)$,
\begin{align}
& \E_{\gK_{t,s}, \gK_{t^\prime,s^\prime}|C_{t,s}, C_{t^\prime,s^\prime},\gH_i}[\noise_{t,s}^\top\noise_{t^\prime,s^\prime}] \nonumber \\
& = \E_{\gK_{t,s}, \gK_{t^\prime,s^\prime}|C_{t,s}, C_{t^\prime,s^\prime},\gH_i}[\frac{1}{C_{t,s}C_{t^\prime,s^\prime}}\sum_{k\in\gK_{t,s}}\sum_{k^\prime\in\gK_{t^\prime,s^\prime}}\noise_{t,s,k}^\top \noise_{t^\prime,s^\prime,k^\prime}] \nonumber\\
& = \frac{1}{C_{t,s}C_{t^\prime,s^\prime}}\sum_{j=1}^{C_{t,s}}\sum_{j^\prime=1}^{C_{t^\prime,s^\prime}}\E_{k\sim \gK_{t,s}|\gH_i}[\noise_{t,s,k}^\top] \E_{k^\prime\sim \gK_{t^\prime,s^\prime}|\gH_i}[\noise_{t^\prime,s^\prime,k^\prime}] \nonumber  \\
& = \frac{1}{C_{t,s}C_{t^\prime,s^\prime}}\sum_{j=1}^{C_{t,s}}\sum_{j^\prime=1}^{C_{t^\prime,s^\prime}}\noise_{s}^\top\noise_{s^\prime} = \noise_{s}^\top\noise_{s^\prime}.
\end{align}
For the case when $(t,s) = (t^\prime, s^\prime)$,
\begin{align}
&\E_{\gK_{t,s}, \gK_{t^\prime,s^\prime}|C_{t,s}, C_{t^\prime,s^\prime},\gH_i}[\noise_{t,s}^\top\noise_{t^\prime,s^\prime}] \nonumber \\
& = \E_{\gK_{t,s}|C_{t,s},\gH_i}[\noise_{t,s}^\top\noise_{t,s}]  \nonumber = \E_{\gK_{t,s}|C_{t,s},\gH_i}[\frac{1}{C_{t,s}^2}\sum_{k,k^\prime\in\gK_{t,s}}\noise_{t,s,k}^\top \noise_{t,s,k^\prime}]  \nonumber \\ 
& = \E_{\gK_{t,s}|C_{t,s},\gH_i}[\frac{1}{C_{t,s}^2}\sum_{k,k^\prime\in\gK_{t,s},k\neq k^\prime}\noise_{t,s,k}^\top\noise_{t,s,k^\prime} + \frac{1}{C_{t,s}^2}\sum_{k\in\gK_{t,s}}\lVert\noise_{t,s,k}\rVert_2^2 ]\nonumber \\
& = \frac{C_{t,s}^2-C_{t,s}}{C_{t,s}^2}\E_{k\sim\gK_{t,s}|\gH_i}[\noise_{t,s,k}^\top]\E_{k^\prime\sim\gK_{t,s}|\gH_i}[\noise_{t,s,k^\prime}] + \frac{1}{C_{t,s}} \E_{k\sim\gK_{t,s}|\gH_i}\lVert\noise_{t,s,k}\rVert_2^2 \nonumber \\
& \leq  \frac{C_{t,s}^2-C_{t,s}}{C_{t,s}^2}\noise_{s}^\top\noise_{s} +  \frac{1}{C_{t,s}} (\lVert\noise_{s}\rVert_2^2 + \rho^2) = \noise_{s}^\top\noise_{s} + \frac{\rho^2}{C_{t,s}}.
\end{align}
\end{proof}

\begin{lemma}
\label{lemma:2}
For any coefficient $x_{i,t} \in \R$,
\begin{align}
\E_{\gK_{:i,:i}|\gH_i}\lVert \sum_{t=1}^i x_{i,t}\sum_{s=1}^i \mW_{[t,s]}\noise_{t,s}\rVert_2^2 \leq \E_{C_{:i, :i}|\gH_i}[\lVert \sum_{t=1}^i x_{i,t}\sum_{s=1}^i \mW_{[t,s]}\noise_{s}\rVert_2^2 + \frac{\rho^2}{C}\sum_{t=1}^i x_{i,t}^2].
\end{align}
\end{lemma}
\begin{proof}
\begin{align}
&\E_{\gK_{:i,:i}|\gH_i}\lVert \sum_{t=1}^i x_{i,t}\sum_{s=1}^i \mW_{[t,s]}\noise_{t,s}\rVert_2^2 \nonumber \\
&= \E_{\gK_{:i,:i}|\gH_i}[ \sum_{\substack{t=1 \\ s=1}}^i\sum_{\substack{t^\prime=1 \\ s^\prime=1}}^i x_{i,t}x_{i, t^\prime}\mW_{[t,s]}\mW_{[t^\prime,s^\prime]}\noise_{t,s}^\top\noise_{t^\prime,s^\prime}] \nonumber \\
&= \E_{C_{:i,:i}|\gH_i}[ \sum_{\substack{t=1 \\ s=1}}^i\sum_{\substack{t^\prime=1 \\ s^\prime=1}}^i x_{i,t}x_{i, t^\prime}\mW_{[t,s]}\mW_{[t^\prime,s^\prime]}\E_{\gK_{t,s}, \gK_{t^\prime,s^\prime}|C_{t,s}, C_{t^\prime,s^\prime},\gH_i}[\noise_{t,s}^\top\noise_{t^\prime,s^\prime}]] \nonumber \\
&\stackrel{\text{Lemma~\ref{lemma:1}}}{\leq} \E_{C_{:i,:i}|\gH_i}[ \sum_{\substack{t=1 \\ s=1}}^i\sum_{\substack{t^\prime=1 \\ s^\prime=1}}^i x_{i,t}x_{i, t^\prime}\mW_{[t,s]}\mW_{[t^\prime,s^\prime]}\noise_{s}^\top\noise_{s^\prime} + \sum_{\substack{t=1 \\ s=1}}^ix_{i,t}^2\mW_{[t,s]}^2\rho^2\frac{1}{C_{t,s}}] \nonumber \\
&= \E_{C_{:i, :i}|\gH_i}[\lVert \sum_{t=1}^i x_{i,t}\sum_{s=1}^i \mW_{[t,s]}\noise_{s}\rVert_2^2 + \frac{\rho^2}{C}\sum_{t=1}^i x_{i,t}^2\sum_{s=1}^i\mW_{[t,s]}] \nonumber\\
&\leq \E_{C_{:i, :i}|\gH_i}[\lVert \sum_{t=1}^i x_{i,t}\sum_{s=1}^i \mW_{[t,s]}\noise_{s}\rVert_2^2 + \frac{\rho^2}{C}\sum_{t=1}^i x_{i,t}^2].
\end{align}
The last inequality comes from the fact that $\sum_{s=1}^t\mW_{[t,s]}\leq 1$.
\end{proof}

\begin{theorem}
For SyncFL and AsyncFL with momentum,
\begin{align}
\E\lVert\frac{1}{T}(\param_{T+1}^\star - \param_{T+1}^\textnormal{sync})\rVert^2_2 &\leq \frac{1}{2}\eta^2TG^2, 
\\
\E\lVert\frac{1}{T}(\param_{T+1}^\star - \param_{T+1}^\textnormal{async})\rVert^2_2 &\leq \eta^2(2TS^2 + TG^2 + 2\frac{\rho^2}{C}).
\end{align}
\end{theorem}

\begin{proof}
For SyncFL with momentum,
\begin{align}
\E\lVert\param_{T+1}^\star - \param_{T+1}^\textnormal{sync}\rVert^2_2 &= \eta^2\E\lVert(\mD - \mD^\star)\mM^\top\mathbf{1}\rVert^2_2 \nonumber \\
& \leq \eta^2T \E\lVert(\mD - \mD^\star)\mM^\top\rVert^2_F \nonumber \\
& = \eta^2T \sum_{i=1}^T \E\lVert\mM_{[i,:]}(\mD - \mD^\star)\rVert^2_2 \nonumber \\
&= \eta^2T \sum_{i=1}^T \E\lVert\sum_{t=1}^i\mM_{[i,t]}\noise_t\rVert^2_2 \nonumber \\
&\leq \eta^2T \sum_{i=1}^T i\sum_{t=1}^i\mM_{[i,t]}^2\E\lVert\noise_t\rVert^2_2 \nonumber \\
&\leq \eta^2T \sum_{i=1}^T iG^2 \leq \frac{1}{2}\eta^2T^3G^2.
\end{align}
By telescoping the constant, we get:
\begin{align}
\E\lVert\frac{1}{T}(\param_{T+1}^\star - \param_{T+1}^\textnormal{sync})\rVert^2_2 &\leq \frac{1}{2}\eta^2TG^2.
\end{align}

For AsyncFL with momentum,
\begin{align}
\E\lVert\param_{T+1}^\star - \param_{T+1}^\textnormal{async}\rVert^2_2 
&= \eta^2\E\lVert\mD^\star(\mM\mW-\mM)^\top\mathbf{1} + \mE\mM^\top\mathbf{1}\rVert^2_2 \nonumber\\
&\leq 2\eta^2(\E\lVert\mD^\star(\mM\mW-\mM)^\top\mathbf{1}\rVert^2_2 + \E\lVert\mE\mM^\top\mathbf{1}\rVert^2_2)\nonumber\\
&\leq 2\eta^2T(\E\lVert\mD^\star(\mM\mW-\mM)^\top\rVert^2_F + \E\lVert\mE\mM^\top\rVert^2_F) \label{eq:async_bound_0}
\end{align}

Let $\check{\mM} = \mM\mW$ where $\check{\mM}_{[i,:]} = \mM_{[i,:]}\mW$ and $\sum_{t=1}^i\check{\mM}_{[i,t]} \leq \sum_{t=1}^i\mM_{[i,t]}$ given $\sum_{s=1}^t\mW_{[t,s]}\leq 1$. 
Then $\lVert\check{\mM}_{[i,:]}\rVert_2\leq \lVert\check{\mM}_{[i,:]}\rVert_1 = \lVert\mM_{[i,:]}\rVert_1\leq 1$. 
We first bound the implicit momentum bias term.

\begin{align}
\E\lVert\mD^\star(\mM\mW-\mM)^\top\rVert^2_F 
&= \sum_{i=1}^T\E\lVert(\check{\mM}_{[i,:]}-\mM_{[i,:]})\mD^{\star\top}\rVert^2_2 \nonumber \\
&=  \sum_{i=1}^T\E\lVert\sum_{t=1}^i(\check{\mM}_{[i,t]}-\mM_{[i,t]})\updates^\star_t\rVert^2_2 \nonumber \\
&\leq \sum_{i=1}^T\E[i\sum_{t=1}^i(\check{\mM}_{[i,t]}-\mM_{[i,t]})^2\lVert\updates^\star_t\rVert^2_2] \nonumber \\
&\leq\sum_{i=1}^TiS^2\E[\lVert\check{\mM}_{[i,:]}-\mM_{[i,:]}\rVert^2_2] \nonumber \\
&\leq S^2\sum_{i=1}^Ti\E[\lVert\check{\mM}_{[i,:]}\rVert^2_2+\lVert\mM_{[i,:]}\rVert_2^2 -2\check{\mM}_{[i,:]}\mM_{[i,:]}^\top] \nonumber \\
&\leq  S^2\sum_{i=1}^T2i \leq S^2T^2 \label{eq:async_bound_1}.
\end{align}

We then bound the asynchronous sampling bias term:
\begin{align}
\E\lVert\mE\mM^\top\rVert_F^2
&= \sum_{i=1}^T \E\lVert \mM_{[i, :]}\mE^\top \rVert_2^2 \nonumber \\
&= \sum_{i=1}^T \E_{\gH_{i}}\E_{\gK_{:i,:i}|\gH_i}\lVert \sum_{t=1}^i\mM_{[i,t]}\sum_{s=1}^i \mW_{[t,s]}\noise_{t,s}\rVert_2^2 \nonumber  \\
&\stackrel{\text{Lemma~\ref{lemma:2}}}{\leq}   \sum_{i=1}^T \E_{\gH_{i}}\E_{C_{1:i,1:i}|\gH_i}[\lVert\sum_{t=1}^i\sum_{s=1}^i \mM_{[i,t]}\mW_{[t,s]}\noise_{s}\rVert_2^2 + \frac{\rho^2}{C}\sum_{t=1}^i \mM_{[i,t]}^2] \nonumber  \\
&=  \sum_{i=1}^T \E[\lVert\sum_{s=1}^i \noise_{s}\sum_{t=1}^i \mM_{[i,t]}\mW_{[t,s]}\rVert_2^2 + \frac{\rho^2}{C}\lVert\mM_{[i:]}\rVert_2^2] \nonumber  \\
&\leq \sum_{i=1}^T \E[\lVert \sum_{s=1}^i \check{\mM}_{[i,s]} \noise_{s} \rVert_2^2 + \frac{\rho^2}{C}] \nonumber  \\
&\leq \sum_{i=1}^T \E[iG^2\lVert\check{\mM}_{[i,:]}\rVert_2^2] + T\frac{\rho^2}{C}\nonumber\\
& \leq \sum_{i=1}^T i G^2 + T\frac{\rho^2}{C} \leq\frac{1}{2}G^2T^2 + T\frac{\rho^2}{C}\label{eq:async_bound_2}.
\end{align}

By substituting~\Cref{eq:async_bound_1} and~(\ref{eq:async_bound_2}) in~\Cref{eq:async_bound_0} and telescoping the constant, we get:
\begin{align}
\E\lVert\frac{1}{T}(\param_{T+1}^\star - \param_{T+1}^\textnormal{async})\rVert^2_2 &\leq \eta^2(2TS^2 + TG^2 + 2\frac{\rho^2}{C}).
\end{align}
\end{proof}

We next present a generalized version of Theorem~\ref{theorem:ma} for $\mW$ with any rank. 
We note that at iteration $i$, the first $i$ elements of solution $\va_i{[:i]} = \mM_{[i,:i]}\mW_{[:i,:i]}^+$ where $\mW_{[:i,:i]}^+$ is the Moore–Penrose inverse of $\mW_{[:i,:i]}$, and the rest elements in $\va_i{[:i]}$ are zeros.
Let $r_i$ denotes the rank of $\mW_{[:i,:i]}$, and let $[\mU_{r_i}, \mU_{-r_i}]\mSigma_i[\mV_{r_i},\mV_{-r_i}]^\top$ be the singular value decomposition of $\mW_{[:i, :i]}$, where $\mU_{r_i}, \mV_{r_i}$ are the first $r_i$ left and right singular vectors ordered by the singular values. 

We define \[\alpha_i = \frac{\mM_{[i,:i]}\mV_{r_i}\mV_{r_i}^\top\mM_{[i,:i]}^\top}{\lVert\mM_{[i,:]}\rVert_2^2} \in [0, 1],\] i.e. the normalized magnitude of projection of $\mM_{[i,:]}$ onto the columns space of $\mW_{[:i,:i]}$. 
Then we can decompose $\lVert\mM_{[i,:]}\rVert_2^2$ as:
\begin{align*}
\lVert\mM_{[i,:]}\rVert_2^2 &= \mM_{[i,:i]}\mV_{r_i}\mV_{r_i}^\top\mM_{[i,:i]}^\top + \mM_{[i,:i]}\mV_{-r_i}\mV_{-r_i}^\top\mM_{[i,:i]}^\top\\
&= \alpha_i \lVert\mM_{[i,:]}\rVert_2^2 + \mM_{[i,:i]}\mV_{-r_i}\mV_{-r_i}^\top\mM_{[i,:i]}^\top
\end{align*}
Thus,  
\begin{align}
\mM_{[i,:i]}\mV_{-r_i}\mV_{-r_i}^\top\mM_{[i,:i]}^\top=(1-\alpha_i)\lVert\mM_{[i,:i]}\rVert_2^2 \leq 1-\alpha_i.
\end{align}

\begin{theorem}
Let $A(T) = \sum_{i=1}^Ti\E[1-\alpha_i]$. 
For AsyncFL with momentum approximation (MA), 
\begin{align}
\E\lVert\frac{1}{T}(\param_{T+1}^\star - \param_{T+1}^\textnormal{MA})\rVert^2_2 
\leq \eta^2(2S^2\frac{A(T)}{T} + G^2T + 2\frac{\rho^2}{TC}\E\lVert\mA\rVert_F^2).
\end{align}
\end{theorem}

\begin{proof}
Similar to~\Cref{eq:async_bound_0}, we have:
\begin{align}
\E\lVert\param_{T+1}^\star - \param_{T+1}^\textnormal{MA}\rVert^2_2 
\leq 2\eta^2T(\E\lVert\mD^\star(\mA\mW-\mM)^\top\rVert^2_F + \E\lVert\mE\mA^\top\rVert^2_F) \label{eq:ma_bound_0}
\end{align}

Denote $\tilde{\mM} = \mA\mW$ and $\tilde{\mM}_{[i,:i]} = \va_i[:i]^\top\mW_{[:i,:i]} = \mM_{[i,:i]}\mW_{[:i,:i]}^+\mW_{[:i,:i]}$.
We first bound the implicit momentum bias term:
\begin{align}
\E\lVert\mD^\star(\mA\mW-\mM)^\top\rVert^2_F 
&= \sum_{i=1}^T\E\lVert(\tilde{\mM}_{[i,:]}-\mM_{[i,:]})\mD^{\star\top}\rVert^2_2 \nonumber \\
&\leq  \sum_{i=1}^TiS^2\E\lVert\tilde{\mM}_{[i,:i]}-\mM_{[i,:i]}\rVert^2_2 \nonumber\\
& = \sum_{i=1}^TiS^2\E\lVert\mM_{[i,:i]}(\mW_{[:i,:i]}^+\mW_{[:i,:i]} - \mI)\rVert^2_2 \nonumber \\
& = \sum_{i=1}^TiS^2\E\lVert\mM_{[i,:i]}\mV_{-r_i}\mV_{-r_i}^\top\rVert^2_2 \nonumber\\
& = \sum_{i=1}^TiS^2\E[\mM_{[i,:i]}\mV_{-r_i}\mV_{-r_i}^\top\mM_{[i,:i]}^\top] \nonumber \\
& \leq \sum_{i=1}^TiS^2\E[1-\alpha_i] = S^2A(T)  \label{eq:ma_bound_1}
\end{align}

We then bound the asynchronous sampling bias term.
Since $\mW^+\mW$ is a orthogonal projection operator, $\lVert\tilde{\mM}_{[i,:]}\rVert_2^2\leq \lVert\mM_{[i,:]}\rVert_2^2\leq 1$ and we have:

\begin{align}
\E\lVert\mE\mA^\top\rVert_F^2
&= \sum_{i=1}^T \E\lVert \mA_{[i, :]}\mE^\top \rVert_2^2 \nonumber \\
&= \sum_{i=1}^T \E_{\gH_{i}}\E_{\gK_{:i,:i}|\gH_i}[\lVert \sum_{t=1}^i\mA_{[i,t]}\sum_{s=1}^i \mW_{[t,s]}\noise_{t,s}\rVert_2^2] \nonumber  \\
&\stackrel{\text{Lemma~\ref{lemma:2}}}{\leq}  \sum_{i=1}^T \E[\lVert \sum_{s=1}^i \tilde{\mM}_{[i:]}\noise_{s}\rVert_2^2 + \frac{\rho^2}{C}\lVert\va_{i}\rVert_2^2] \nonumber  \\
&\leq \sum_{i=1}^T \E[i G^2\lVert\tilde{\mM}_{[i,:]}\rVert_2^2] + \frac{\rho^2}{C}\sum_{i=1}^T\E\lVert\va_{i}\rVert_2^2 \nonumber \\
&\leq \sum_{i=1}^T i G^2 + \frac{\rho^2}{C}\E\lVert\mA\rVert_F^2 \leq \frac{1}{2}T^2G^2 + \frac{\rho^2}{C}\E\lVert\mA\rVert_F^2.
\label{eq:ma_bound_2}
\end{align}

By substituting~\Cref{eq:ma_bound_1} and~(\ref{eq:ma_bound_2}) in~\Cref{eq:ma_bound_0} and telescoping the constant, we get:
\begin{align}
\E\lVert\frac{1}{T}(\param_{T+1}^\star - \param_{T+1}^\textnormal{MA})\rVert^2_2 
\leq \eta^2(2S^2\frac{A(T)}{T} + G^2T + 2\frac{\rho^2}{TC}\E\lVert\mA\rVert_F^2).
\end{align}
\end{proof}

When $\mW$ is full rank, $\alpha_i = 1$ for all iterations and $A(T) = 0$. 
When $\mW$ is not full rank, $\alpha_i$ is more likely to be closer to 1 when $r_i$ is high and leads to smaller $A(T)$. 
We discuss in Appendix~\ref{app:discuss} that how can down-scaling factor increases $r_i$ and reduces implicit momentum bias.

\subsection{Discussion}
\label{app:discuss}

\paragraph{Justification for Assumption~\ref{assumption:random}.}
We argue that in the production FL systems, the order of sampled clients' updates arriving at server is random does not dependent on their data size. 
In common deployed FL systems~\cite{bonawitz2019towards,paulik2021federated,huba2022papaya,wang2023flint}, the initiation of on-device training process is subject to a set of conditions: connected to power and wireless network, idle, and scheduled by device OS.
The timing of the event for all these conditions to be met is naturally nondeterministic rather than data dependent, e.g. a client might always charge their device at certain time of a day but OS scheduler might not always prioritize the FL training process and start training at exactly the time of charging.

Let $X_k$ denote the random variable for client $k$'s the starting time of training, and let $Y_k$ be the random variable of on-device training and network latency for submitting model update to server.
Then $X_k + Y_k$ is the time that $k$'s updates arrived at server.
When the variance of $X_k$ is dominating $X_k+Y_k$, the arrival order of sampled clients 
is not data dependent and random.
This is highly likely as the on-device training becomes extremely efficient with advances in hardware, e.g. training a modern neural network on an edge device takes only a few seconds~\cite{lin2020mcunet,lin2022device,cai2022enable}. 
To further enforce the arrival order to be random and less data dependent, we can also implement simple on-device logic such as enforcing the maximum amount of data to train on and injects a small random delay before on-device training~\cite{talwar2023samplable}.

\paragraph{Importance of the down-scaling factor.}
The down-scaling factor $(t-s+1)^{-p}$ for the stale model updates plays an important role when solving~\Cref{eq:obj}. 
We find that larger $p$ leads to higher rank of $\mW$ and smaller least square objective, and thus better momentum approximation as illustrated in Figure~\ref{fig:discuss}.
For $p=0.0$, the nullity ($i-r_i$) increases over iterations while $\mW_{[:i, :i]}$ is mostly full-rank when choosing a higher $p$. 
As a consequence, the relative least square error $1-\alpha_i$ is nearly zero for higher $p$.
However, we cannot set $p$ arbitrarily high as this would over-penalize the stale gradients and impact the model convergence.

\begin{figure}
\centering
\includegraphics[height=0.265\linewidth]{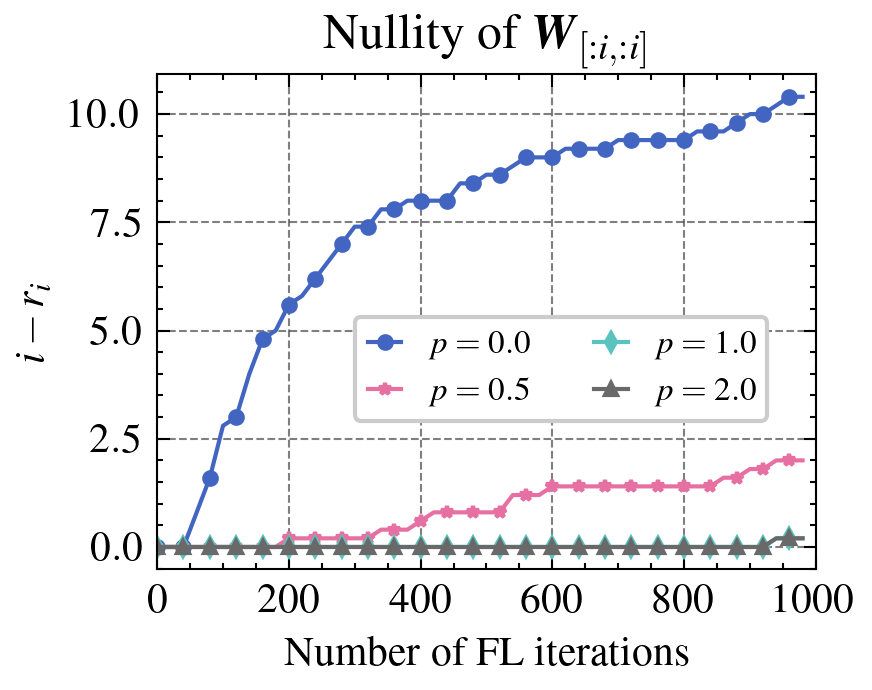}
\includegraphics[height=0.265\linewidth]{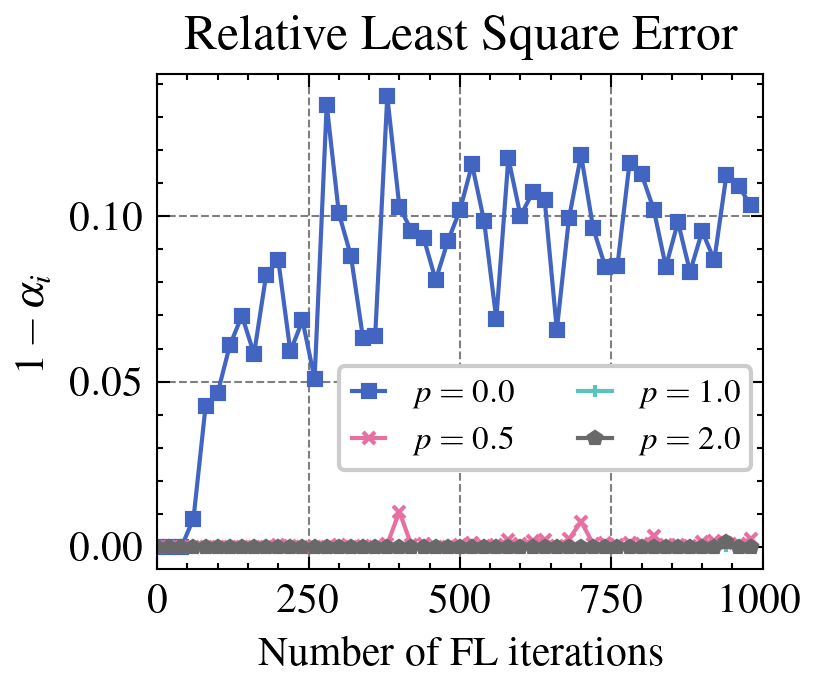}
\includegraphics[height=0.265\linewidth]{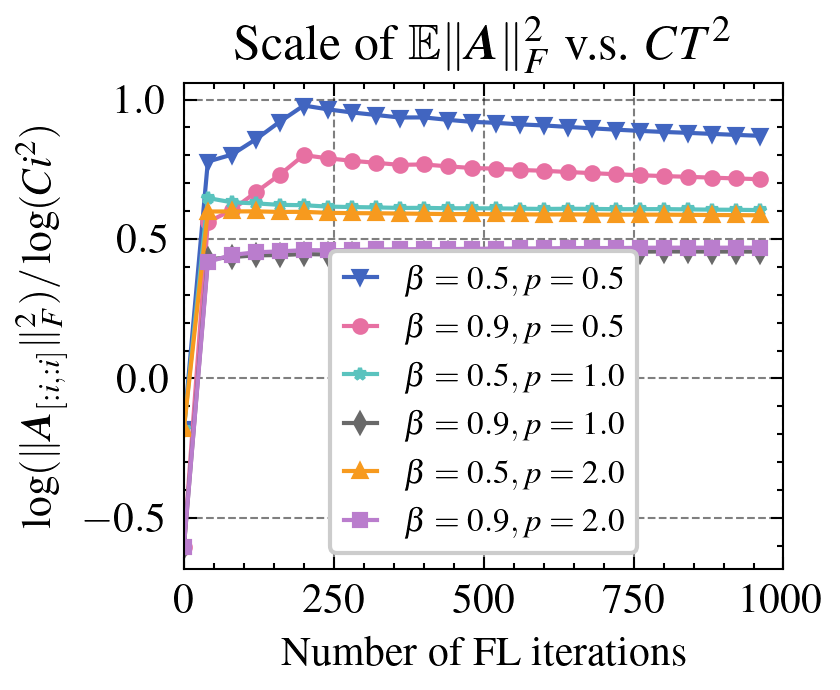}
\caption{(Left) Nullity ($i-r_i$) of $\mW_{[:i, :i]}$ over iterations for different $p$.
(Middle) $1-\alpha_i$ over iterations for different $p$. 
(Right) Scale of $\E\lVert\mA\rVert_F^2$ versus $CT^2$ in the log space over iterations for different $p$ and $\beta$.}
\label{fig:discuss}
\end{figure}

\paragraph{Scale of $\E\lVert\mA\rVert_F^2$.}
We empirically evaluate the condition in Theorem~\ref{theorem:ma} that $\E\lVert\mA\rVert_F^2=\gO(CT^2)$, i.e. $\E\lVert\mA\rVert_F^2$ should not grow faster than $CT^2$ when $T$ increases. 
Right of Figure~\ref{fig:discuss} shows the quantity $\log\lVert\mA_{[:i,:i]}\rVert_F^2 / \log Ci^2$ over iterations, where the quantity is less than 1 for different choices of $p$ and $\beta$, indicating that  $\E\lVert\mA\rVert_F^2$ grows slower than $CT^2$.
We also note that higher $p$ results in smaller $\E\lVert\mA\rVert_F^2$ leading to smaller sampling error in~\Cref{eq:ma_bound_main}.

\section{Limitations}
\label{app:limitation}
Our analysis in Theorem~\ref{theorem:async} and Theorem~\ref{theorem:ma} did not make any assumptions about the distribution of $\mW$. 
Though the analysis gives the results for an arbitrary $\mW$ but we note that the distribution of $\mW$ might have some special properties when using a particular client delay distribution. 
As discussed in Appendix~\ref{app:discuss}, we also empirically find the spectral properties of $\mW$ are associated the choice of down-scaling factor and its exponent $p$. 
These special properties of $\mW$, which we did not formalize, could potentially improve the theoretical results . 
We leave it to future work to explore the impact of the distribution and properties of $\mW$ on momentum in AsyncFL. 
We also did not analyze the impact of stale gradients on the bias in the second moments in optimizers like FedAdam as we acknowledged in Section~\ref{sec:method}.

\section{Broader Impact}
\label{app:impact}
This work does not have negative societal or ethical impact. 
On the contrary, this work can potentially benefit the society in terms of stronger privacy protection.
Our proposed method is compatible with secure aggregation and differential privacy, and can be easily integrated to existing asynchronous federated learning production systems.
We believe that our method can improve the applicability of asynchronous private federated learning to more on-device ML products where the data is highly personal and sensitive, and thus provide meaningful privacy guarantee to the end users.